\documentclass{article}

\usepackage[preprint]{neurips_2020}

\usepackage[utf8]{inputenc} 
\usepackage[T1]{fontenc}    
\usepackage{hyperref}       
\usepackage{url}            
\usepackage{booktabs}       
\usepackage{amsfonts}       
\usepackage{nicefrac}       
\usepackage{microtype}      
\usepackage{amsmath}
\usepackage{amsthm}
\usepackage{mathrsfs}
\usepackage{graphicx}

\newtheorem{defn}{\textsc{\textbf{Definition}}}
\newtheorem{thm}{\textsc{\textbf{Theorem}}}
\newtheorem{cor}{\textsc{\textbf{Corollary}}}
\newtheorem{lem}{\textsc{\textbf{Lemma}}}

\title{On Influence Functions, Classification Influence, Relative Influence, Memorization and Generalization}

\author{%
  Michael Kounavis \\
  Meta Platforms Inc.\\
  New York, NY \\
  \texttt{michaelkounavis@meta.com} \\
  \And
  Ousmane Dia\\
  Meta Platforms Inc.\\
  Bellevue, WA \\
  \texttt{ousamdia@meta.com} \\
  \And
  Ilqar Ramazanli\\
  Meta Platforms Inc.\\
  New York, NY \\
  \texttt{iramazanli@meta.com} \\
}

\begin{document}

\maketitle

\begin{abstract}
  Machine learning systems such as large scale recommendation systems or natural language processing systems are usually trained on billions of training points and are associated with hundreds of billions or trillions of parameters. Improving the learning process in such a way that both the training load is reduced and the model accuracy improved is  highly desired. In this paper we take a first step toward solving this problem, studying influence functions from the perspective of simplifying the computations they involve. We discuss assumptions, under which influence computations can be performed on significantly fewer parameters. We also demonstrate that the sign of the influence value can indicate whether a training point is to memorize, as opposed to generalize upon. For this purpose we formally define what memorization means for a training point, as opposed to generalization. We conclude that influence functions can be made practical, even for large scale machine learning systems, and that influence values can be taken into account by algorithms that selectively remove training points, as part of the learning process.
\end{abstract}

\section{Introduction}

Influence functions [1][6] are analytical tools that offer measures of the impact of training points on the predictions a neural network makes. The formulas coming from the standard influence function theory involve differentiation and double differentiation over the many billions of parameters of a neural network. Furthermore, differentiation and double differentiation are performed on all training points. In this document we study several simplifications to this process analytically, which significantly reduce the complexity of these computations. Among others:

\begin{itemize}
    \item The neural network model is split into two parts: a heavyweight featurizer and  a lightweight classifier. Influence functions are computed only on the parameters of the classifier.
    \item Double differentiation is performed only on a small number of training points which are sampled.
    \item The focus of the study is on the sign of the influence values. This can be taken into account, for instance, in order to sample training points.
\end{itemize}

The aim of our work is to investigate design principles for machine learning systems where not only the training load is reduced, but also the accuracy improved. For this purpose we explore some properties of influence functions, not fully investigated before. Our findings are summarized as follows:

\begin{itemize}
    \item Influence functions computed on the parameters of a classifier can be good approximations of the standard theory influence functions. This is true when the eigenvalues of the inverse Hessian, associated with the classifier parameters, are the most dominant ones from among all inverse Hessian eigenvalues,  and the same applies to the corresponding eigenvectors.
    \item There exist inverse Hessian substitutes, computed from small subsets of the training data set, which result in the computation of influence values that qualitatively carry the same information as the influence values computed from the inverse Hessians coming from the entire data set.
    \item The sign of the influence value indicates whether a training point is to memorize or not. We formally define what it means for a training point to be memorizable, as opposed to being a point to generalize upon. Thus, keeping the training points to memorize and sampling from the training points to generalize is a valid sampling methodology.
\end{itemize}

\section{Preliminaries and notation}

We consider a data set $\mathcal{Z}$, as a set containing $n$ training or test points $z_i$ each point consisting of a feature vector $x_i$ and a label $y_i$, all of which contain/are real numbers:

\begin{equation}
\label{e:eq1}
\mathcal{Z} = \{(x_i, y_i), i \in [1,n]\}
\end{equation}

We also consider a trained model associated with a parameter set $\theta$. The optimal parameter set $\hat{\theta}$ is the one that minimizes the empirical risk:

\begin{equation}
\label{e:eq2}
R(\theta)= \displaystyle \frac {1}{n} \displaystyle \sum_{i = 1}^{n} {\mathcal{L}(z_i, \theta)}
\end{equation}

for some given loss function $\mathcal{L}()$. Specifically:

\begin{equation}
\label{e:eq3}
\hat{\theta} = \underset{\theta}{\operatorname{argmin}} \> R(\theta)
\end{equation}

The Hessian matrix associated with parameter set
$\hat{\theta}$ and loss function $\mathcal{L}()$ is the symmetric matrix:

\begin{equation}
\label{e:eq4}
H_{\hat{\theta}} = \displaystyle \frac {1}{n} \displaystyle \sum_{i = 1}^{n} {\nabla_{\theta}^2 \> \mathcal{L}(z_i, \hat{\theta})}
\end{equation}

which is considered positive definite and thus invertible. If we assume that there are $\mu$ parameters in the set $\theta$ i.e., $\theta = (\theta_1, ...,  \theta_\mu)^\top$, then the Hessian matrix can also be written as:

\begin{equation}
\label{e:eq5}
H_{\hat{\theta}} = \displaystyle \frac {1}{n} \displaystyle \sum_{i = 1}^{n} {\begin{pmatrix} \displaystyle \frac{\partial^2 \mathcal{L}(z_i, \hat{\theta})} {\partial \theta_1^2} & \cdots & \displaystyle \frac{\partial^2 \mathcal{L}(z_i, \hat{\theta})} {\partial \theta_1 \partial \theta_\mu}\\ \vdots & \ddots & \vdots \\ \displaystyle \frac{\partial^2 \mathcal{L}(z_i, \hat{\theta})} {\partial \theta_\mu \partial \theta_1} & \cdots & \displaystyle \frac{\partial^2 \mathcal{L}(z_i, \hat{\theta})} {\partial \theta_\mu^2} \end{pmatrix}}
\end{equation}

The Hessian matrix is a function of the model parameters.

\section{Standard influence function theory}

We summarize the main points from the work of Koh et. al. [1]. Influence of a training point $z_p$ is defined with respect to a modified empirical risk $R(\theta) + \epsilon \> \mathcal{L}(z_p, \theta)$, in which the loss term associated with point $z_p$ is
multiplied with weight $1 + \epsilon$. The parameter set minimizing this empirical risk is defined as:

\begin{equation}
\label{e:eq6}
\hat{\theta}_{\epsilon, z_p} = \underset{\theta}{\operatorname{argmin}} \> (R(\theta) + \epsilon \> \mathcal{L}(z_p, \theta))
\end{equation}

If we set $\epsilon \leftarrow -\frac{1}{n}$ in relation (\ref{e:eq6}), then we get the empirical risk minimizer associated with a training data set in which the point $z_p$ is absent. In what follows we derive the main influence function formulas, paying attention to the assumptions and simplification introduced in reference [1]. We first set $M(\epsilon) \leftarrow \hat{\theta}_{\epsilon, z_p} $, a function of the scalar $\epsilon$ for fixed $z_p$. Using the Maclaurin expansion formula, we get:

\begin{equation}
\label{e:eq7}
M(\epsilon) = M(0) + M(0)' \cdot \epsilon + \displaystyle\frac{M(0)''}{2!} \cdot \epsilon^2+ \ldots
\end{equation}

From (\ref{e:eq7}), the change in optimal model parameters when training point $z_p$ is multiplied with weight $1 + \epsilon$ is given by:

\begin{equation}
\label{e:eq8}
M(\epsilon) - M(0) = M(0)' \cdot \epsilon + \displaystyle\frac{M(0)''}{2!} \cdot \epsilon^2 + \ldots
\end{equation}

We further simplify (\ref{e:eq8}) by setting:

\begin{equation}
\label{e:eq9}
T_1 \leftarrow \> \displaystyle \frac{M(0)''}{2!} \cdot \epsilon^2 + \ldots
\end{equation}

where term $T_1 = T_1(\epsilon)$ is also a function of $\epsilon$. The change in optimal model parameters becomes:

\begin{equation}
\label{e:eq10a}
\hat{\theta}_{\epsilon, z_p} - \hat{\theta} = M(0)' \cdot \epsilon + T_1=\displaystyle \frac{\partial \hat{\theta}_{\epsilon, z_p}}{\partial \epsilon}\big|_{\> \epsilon\leftarrow 0} \> \cdot \epsilon + T_1
\end{equation}

where $\big|_{\> x \leftarrow v}$ denotes assignment of value $v$ to variable $x$ and:

\begin{equation}
\label{e:eq10}
\hat{\theta}_{-\frac{1}{n}, z_p} - \hat{\theta} =\displaystyle \frac{\partial \hat{\theta}_{\epsilon, z_p}}{\partial \epsilon}\big|_{\> \epsilon\leftarrow 0} \> \cdot (-\displaystyle \frac{1}{n}) + T_1\big|_{\epsilon \leftarrow - \frac{1}{n}}
\end{equation}

We proceed with the derivation of the main influence function formula by observing that, since $\hat{\theta}_{\epsilon, z_p} $ is the minimizer of the empirical risk
$R(\theta) + \epsilon \> \mathcal{L}(z_p, \theta)$, it holds that:

\begin{equation}
\label{e:eq11}
\nabla\big(R(\hat{\theta}_{\epsilon, z_p}) + \>\epsilon \> \mathcal{L}(z_p, \hat{\theta}_{\epsilon, z_p})\big) = 0
\end{equation}

Applying the Taylor expansion formula on relation (\ref{e:eq11}), with $\hat{\theta}$ being treated as constant and $\hat{\theta}_{\epsilon, z_p}$ being treated as the value of the independent variable, and simplifying
$\nabla_{\theta}$ as $\nabla$ we get:

\begin{equation}
\label{e:eq12}
\begin{array}{c}
\nabla \big(R(\hat{\theta})+ \nabla R(\hat{\theta})\cdot(\hat{\theta}_{\epsilon, z_p}-\hat{\theta}) +\displaystyle\frac{1}{2}(\hat{\theta}_{\epsilon, z_p}-\hat{\theta})^\top\cdot\nabla^2R(\hat{\theta})\cdot(\hat{\theta}_{\epsilon, z_p}-\hat{\theta}) \>+T_2'\> +
\vspace{4pt}\\
\epsilon\>\mathcal{L}(z_p, \hat{\theta}) + \epsilon \> \nabla\mathcal{L}(z_p, \hat{\theta})\cdot(\hat{\theta}_{\epsilon, z_p}-\hat{\theta})+\displaystyle\frac{\epsilon}{2} (\hat{\theta}_{\epsilon, z_p}-\hat{\theta})^\top\cdot \nabla^2\mathcal{L}(z_p, \hat{\theta})\cdot(\hat{\theta}_{\epsilon, z_p}-\hat{\theta})\>+T_3'\>\big)=0
\vspace{4pt}\\
\end{array}
\end{equation}

where $T_2'$ and $T_3'$ are Taylor expansion terms or order higher than 2. Next we set:

\begin{equation}
\label{e:eq13}
T_2 = -\nabla \> T_2', \>\> T_3=-\nabla \> T_3'
\end{equation}

From relations (\ref{e:eq12}) and (\ref{e:eq13}), the fact that $\nabla R(\hat{\theta})=0$, we obtain:

\begin{equation}
\label{e:eq14}
\begin{array}{c}
\nabla^2 R(\hat{\theta})\cdot(\hat{\theta}_{\epsilon, z_p}-\hat{\theta}) - T_2 \> + \epsilon \> \nabla\mathcal{L}(z_p, \hat{\theta}) \> +\epsilon \> \nabla^2\mathcal{L}(z_p, \hat{\theta})\cdot(\hat{\theta}_{\epsilon, z_p}-\hat{\theta})\>-T_3 = 0 \>\> \Rightarrow
\vspace{4pt}\\
\nabla^2 R(\hat{\theta})\cdot \displaystyle\frac{\hat{\theta}_{\epsilon, z_p}-\hat{\theta}}{\epsilon} = -\nabla\mathcal{L}(z_p, \hat{\theta}) - \nabla^2\mathcal{L}(z_p, \hat{\theta})\cdot(\hat{\theta}_{\epsilon, z_p}-\hat{\theta})\>+\displaystyle \frac{T_2 + T_3}{\epsilon}
\end{array}
\end{equation}

Finally we compute the limit of both sides of (\ref{e:eq14}) when $\epsilon \rightarrow 0$, considering that $H_{\hat{\theta}}=\nabla^2R(\hat{\theta})$.
The above relation becomes :

\begin{equation}
\label{e:eq15}
\begin{array}{c}
H_{\hat{\theta}}\> \cdot\displaystyle\frac{\partial \hat{\theta}_{\epsilon, z_p}}{\partial \epsilon} \big|_{\epsilon \leftarrow 0}=-\nabla\mathcal{L}(z_p, \hat{\theta})+ \lim_{\epsilon \rightarrow0}\displaystyle\frac{T_2+T_3}{\epsilon} \> \Rightarrow
\vspace{4pt}\\
\displaystyle\frac{\partial \hat{\theta}_{\epsilon, z_p}}{\partial \epsilon} \big|_{\epsilon \leftarrow 0}=- H_{\hat{\theta}}^{-1} \nabla\mathcal{L}(z_p, \hat{\theta})+ \lim_{\epsilon \rightarrow 0} H_{\hat{\theta}}^{-1} \cdot \displaystyle\frac{T_2+T_3}{\epsilon}
\end{array}
\end{equation}

The expression of the right hand side of relation (\ref{e:eq15}) is  the influence function $I_{up}(\hat{\theta}, z_p)$.
Combining (\ref{e:eq10a}), (\ref{e:eq10}) and (\ref{e:eq15}), we obtain an expression for the change in optimal parameters:

\begin{equation}
\label{e:eq16}
\begin{array}{l}
\hat{\theta}_{\epsilon, z_p} - \hat{\theta} = - \epsilon \cdot H_{\hat{\theta}}^{-1} \nabla\mathcal{L}(z_p, \hat{\theta})+ \epsilon \cdot \big(\lim_{\epsilon \rightarrow 0} H_{\hat{\theta}}^{-1} \cdot \displaystyle\frac{T_2+T_3}{\epsilon}\big) + T_1
\vspace{4pt}\\
\hspace{39pt}= \epsilon \cdot I_{up}(\hat{\theta}, z_p) + T_1
\vspace{4pt}\\
\hat{\theta}_{-\frac{1}{n}, z_p} -\hat{\theta} = \displaystyle\frac{1}{n} \cdot H_{\hat{\theta}}^{-1} \nabla\mathcal{L}(z_p, \hat{\theta}) - \displaystyle\frac{1}{n} \cdot \big(\lim_{\epsilon \rightarrow 0} H_{\hat{\theta}}^{-1} \cdot \displaystyle\frac{T_2+T_3}{\epsilon}\big) + T_1 \big|_{\epsilon \leftarrow - \frac{1}{n}}
\vspace{4pt}\\
\hspace{49pt} = -\displaystyle\frac{1}{n} \cdot I_{up}(\hat{\theta}, z_p) + T_1\big|_{\epsilon \leftarrow -\frac{1}{n}}
\vspace{4pt}\\
\end{array}
\end{equation}

\begin{defn}
\label{d:defn1}
(Negligible additive term). An $\mu$-dimensional additive term $\delta \in \mathbb{R}^{\mu}$ is called negligible according to bound $\varepsilon$, norm $\mathbb{L}_q$, function $f$, and input set $X \subseteq \mathbb{R}^{\mu}$ if and only if, for any $x \in X$, it holds that:

\begin{equation}
\label{e:eq17}
|\|f(x+\delta)\|_q-\|f(x)\|_q| \leq \varepsilon
\end{equation}

\end{defn}

In the analysis that follows, and for the sake of simplicity, we will be omitting negligible terms from the relations they are part of. Omitting negligible terms means that these terms must satisfy relation (\ref{e:eq17}) for some $f, \mathbb{L}_q, \epsilon$. The specific $f, \mathbb{L}_q, \epsilon$, for which terms are negligible are also omitted or implied.

\vspace{4pt}
\begin{cor}
\label{c:cor1}
(Simplification of the expressions for the influence function $I_{up}()$ and the model parameter
change $\hat{\theta}_{-\frac{1}{n}, z_p}-\hat{\theta}$).
If the term
$\lim_{\epsilon \rightarrow 0} H_{\hat{\theta}}^{-1} \cdot \displaystyle\frac{T_2+T_3}{\epsilon}$
is negligible according to some bound, norm and input set then:

\begin{equation}
\label{e:eq18}
I_{up}(\hat{\theta}, z_p) =  - H_{\hat{\theta}}^{-1} \nabla\mathcal{L}(z_p, \hat{\theta})
\end{equation}

Furthermore, if the term
$T_1\big|_{\epsilon \leftarrow -\frac{1}{n}} - \displaystyle\frac{1}{n} \cdot \big(\lim_{\epsilon \rightarrow 0} H_{\hat{\theta}}^{-1} \cdot \displaystyle\frac{T_2+T_3}{\epsilon}\big)$
 is negligible then:

\begin{equation}
\label{e:eq19}
\hat{\theta}_{-\frac{1}{n}, z_p} -\hat{\theta}=  \displaystyle\frac{1}{n} \cdot  H_{\hat{\theta}}^{-1} \nabla\mathcal{L}(z_p, \hat{\theta})
\end{equation}
\end{cor}

A second type of influence function measures the impact on the loss experienced at a test point
$z_t=(x_t, y_t)$ when the weight that multiplies the loss component associated with training point
$z_p = (x_p, y_p)$ becomes
$1 + \epsilon$:

\begin{equation}
\label{e:eq20}
\Delta\mathcal{L}_{z_t, z_p}(\hat{\theta}, \epsilon)=\mathcal{L}(z_t, \hat{\theta}_{\epsilon, z_p})-\mathcal{L}(z_t, \hat{\theta})
\end{equation}

When $\epsilon = -\displaystyle\frac{1}{n}$, relation (\ref{e:eq20}) expresses the change in the loss at point $z_t$ when training point $z_p$ is absent. We proceed in the same manner as with the derivation of the formula for $\hat{\theta}_{\epsilon, z_p}-\hat{\theta}$, performing Taylor expansion on the right hand side of relation (\ref{e:eq20}):

\begin{equation}
\label{e:eq21}
\begin{array}{c}
\mathcal{L}(z_t, \hat{\theta}_{\epsilon, z_p})-\mathcal{L}(z_t, \hat{\theta}) = \mathcal{L}(z_t, \hat{\theta}) + \nabla\mathcal{L}(z_t, \hat{\theta})^\top\cdot(\hat{\theta}_{\epsilon, z_p} - \hat{\theta}) + T_4 - \mathcal{L}(z_t, \hat{\theta}) \Leftrightarrow
\vspace{4pt}\\
\mathcal{L}(z_t, \hat{\theta}_{\epsilon, z_p})-\mathcal{L}(z_t, \hat{\theta}) =  \nabla\mathcal{L}(z_t, \hat{\theta})^\top\cdot(\hat{\theta}_{\epsilon, z_p} - \hat{\theta}) + T_4
\end{array}
\end{equation}

where $T_4$ is the sum of Taylor expansion terms of order two or higher.  Substituting $\hat{\theta}_{\epsilon, z_p}-\hat{\theta}$ with the right hand side of the first of the relations (\ref{e:eq16}) we obtain:

\begin{equation}
\label{e:eq22}
\begin{array}{l}
\mathcal{L}(z_t, \hat{\theta}_{\epsilon, z_p})-\mathcal{L}(z_t, \hat{\theta}) =
\vspace{4pt}\\
\hspace{42pt}
- \epsilon \cdot\nabla\mathcal{L}(z_t, \hat{\theta})^\top H_{\hat{\theta}}^{-1}\nabla\mathcal{L}(z_p, \hat{\theta}) +  \epsilon \cdot \nabla \mathcal{L}(z_t, \hat{\theta})^\top \big(\lim_{\epsilon \rightarrow 0} H_{\hat{\theta}}^{-1} \cdot \displaystyle\frac{T_2+T_3}{\epsilon}\big)
\vspace{4pt}\\
\hspace{42pt}
+ \nabla \mathcal{L}(z_t, \hat{\theta})^\top \cdot T_1 +T_4
\end{array}
\end{equation}

Moreover, when $\epsilon = -\displaystyle\frac{1}{n}$, the change in loss
$\mathcal{L}(z_t, \hat{\theta}_{-\frac{1}{n}, z_p})-\mathcal{L}(z_t, \hat{\theta})$
experienced at $z_t$ when $z_p$ is removed from the training data set is given by:

\begin{equation}
\label{e:eq23}
\begin{array}{l}
\mathcal{L}(z_t, \hat{\theta}_{-\frac{1}{n}, z_p})-\mathcal{L}(z_t, \hat{\theta})  =
\vspace{4pt}\\
\hspace{42pt}
\displaystyle\frac{1}{n}  \nabla\mathcal{L}(z_t, \hat{\theta})^\top H_{\hat{\theta}}^{-1}\nabla\mathcal{L}(z_p, \hat{\theta}) - \displaystyle\frac{1}{n} \nabla \mathcal{L}(z_t, \hat{\theta})^\top \big(\lim_{\epsilon \rightarrow 0} H_{\hat{\theta}}^{-1} \cdot \displaystyle\frac{T_2+T_3}{\epsilon}\big)
\vspace{4pt}\\
\hspace{42pt}
+ \nabla \mathcal{L}(z_t, \hat{\theta})^\top \cdot T_1\big|_{\epsilon\leftarrow-\frac{1}{n}} +T_4\big|_{\epsilon\leftarrow-\frac{1}{n}}
\end{array}
\end{equation}

The entity
$\displaystyle\frac{\mathcal{L}(z_t, \hat{\theta}_{\epsilon, z_p})-\mathcal{L}(z_t, \hat{\theta})}{\epsilon}$
is an influence function and is denoted by
$I_{loss}(\epsilon, \hat{\theta}, z_t, z_p)$. An expression for this loss function is given from relation (\ref{e:eq22}):

\begin{equation}
\label{e:eq24}
\begin{array}{l}
I_{loss}(\epsilon, \hat{\theta}, z_t, z_p) =
\vspace{4pt}\\
\hspace{52pt}
- \nabla\mathcal{L}(z_t, \hat{\theta})^\top H_{\hat{\theta}}^{-1}\nabla\mathcal{L}(z_p, \hat{\theta}) +
\vspace{4pt}\\
\hspace{52pt}
\nabla \mathcal{L}(z_t, \hat{\theta})^\top \big(\lim_{\epsilon \rightarrow 0} H_{\hat{\theta}}^{-1} \cdot \displaystyle\frac{T_2+T_3}{\epsilon}\big) + \displaystyle\frac{1}{\epsilon}\nabla \mathcal{L}(z_t, \hat{\theta})^\top \cdot T_1 +\displaystyle\frac{T_4}{\epsilon}
\end{array}
\end{equation}

From (\ref{e:eq23}) and (\ref{e:eq24}) it follows that:

\begin{equation}
\label{e:eq25}
\mathcal{L}(z_t, \hat{\theta}_{-\frac{1}{n}, z_p})-\mathcal{L}(z_t, \hat{\theta}) = -\displaystyle\frac{1}{n} I_{loss}(-\frac{1}{n}, \hat{\theta}, z_t, z_p)
\end{equation}

\begin{cor}
\label{c:cor2}
(Simplification of the expressions for the influence function $I_{loss}()$ and the change in loss experienced at $z_t$).
If the term
$\nabla \mathcal{L}(z_t, \hat{\theta})^\top \big(\lim_{\epsilon \rightarrow 0} H_{\hat{\theta}}^{-1} \cdot \displaystyle\frac{T_2+T_3}{\epsilon}\big) + \displaystyle\frac{1}{\epsilon}\nabla \mathcal{L}(z_t, \hat{\theta})^\top \cdot T_1 +\displaystyle\frac{T_4}{\epsilon}$
is negligible according to some bound, norm and input set, then the influence function $I_{loss}(\epsilon, \hat{\theta}, z_t, z_p)$ is independent of the value of $\epsilon$ and given by:

\begin{equation}
\label{e:eq26}
I_{loss}(\hat{\theta}, z_t, z_p) = - \nabla\mathcal{L}(z_t, \hat{\theta})^\top H_{\hat{\theta}}^{-1}\nabla\mathcal{L}(z_p, \hat{\theta})
\end{equation}

Furthermore, if the term
$- \displaystyle\frac{1}{n} \nabla \mathcal{L}(z_t, \hat{\theta})^\top \big(\lim_{\epsilon \rightarrow 0} H_{\hat{\theta}}^{-1} \cdot \displaystyle\frac{T_2+T_3}{\epsilon}\big) + \nabla \mathcal{L}(z_t, \hat{\theta})^\top \cdot T_1\big|_{\epsilon\leftarrow-\frac{1}{n}} +T_4\big|_{\epsilon\leftarrow-\frac{1}{n}}$
is negligible then the change in loss experienced at test point $z_t$ when point $z_p$ is removed from the training data set is given by:

\begin{equation}
\label{e:eq27}
\mathcal{L}(z_t, \hat{\theta}_{-\frac{1}{n}, z_p})-\mathcal{L}(z_t, \hat{\theta}) =  \displaystyle\frac{1}{n}\nabla\mathcal{L}(z_t, \hat{\theta})^\top H_{\hat{\theta}}^{-1}\nabla\mathcal{L}(z_p, \hat{\theta})
\end{equation}
\end{cor}

This completes the presentation of the standard influence theory. The emphasis was on the various assumptions that lead to the simplified formulas of (\ref{e:eq18}), (\ref{e:eq19}), (\ref{e:eq26}) and (\ref{e:eq27}). In the analysis that follows we assume that Corollaries \ref{c:cor1} and \ref{c:cor2} hold.

\begin{figure}
  \centering
  \includegraphics[scale=0.37]{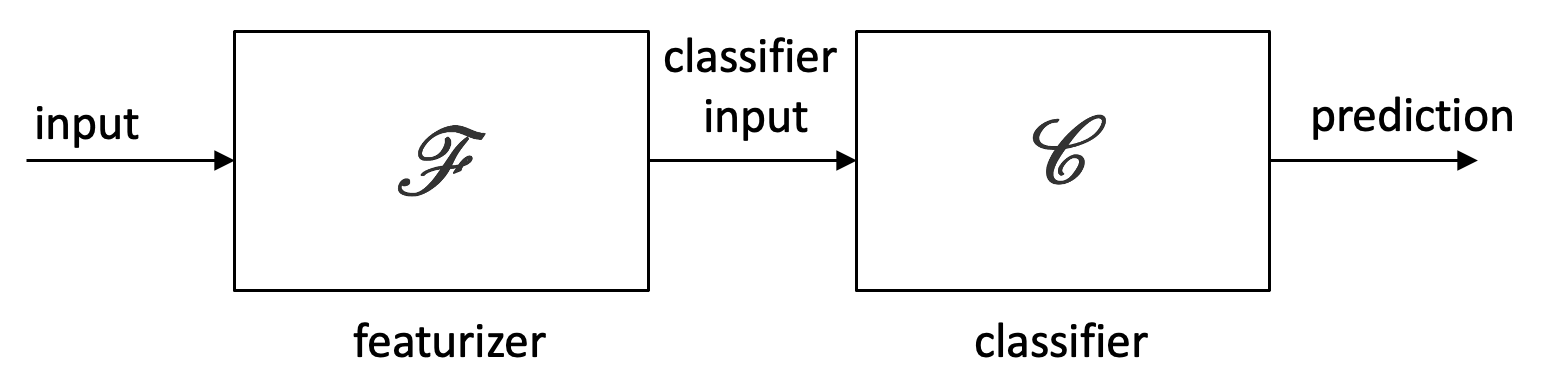}
  \caption{The featurizer-classifier pair}
\label{p:pic1}
\end{figure}

\section{Classification influence}

In this section, we consider a machine learning system consisting of two subsystems: a featurizer $\mathscr{F}()$
and a classifier $\mathscr{C}()$. These are serially connected as shown in Figure \ref{p:pic1}. A prediction $\tilde{y}$ associated with a feature vector $x$  is computed as the output of the classifier, where the corresponding input to the classifier is the output of the featurizer:

\begin{equation}
\label{e:eq28}
\begin{array}{c}
\tilde{y} = \mathscr{C}(\mathscr{F}(x)) = \mathscr{C}(v)
\vspace{4pt}\\
v=\mathscr{F}(x)
\end{array}
\end{equation}

We consider that the model parameter set $\theta$ is the union of two disjoint sets. One is the set of parameters of the featurizer $\theta_{\mathscr{F}}$ and the other the set of parameters of the classifier $\theta_{\mathscr{C}}()$. Similarly the set of parameter values $\hat{\theta}$ that minimizes the empirical risk of relation (\ref{e:eq2}) is the union of two disjoint sets of optimal parameter values: One is associated with the featurizer $\hat{\theta}_{\mathscr{F}}$ and one is associated with the classifier $\hat{\theta}_{\mathscr{C}}$:

\begin{equation}
\label{e:eq29}
\theta = \theta_{\mathscr{F}} \cup \theta_{\mathscr{C}},
\hspace{6pt}
\hat{\theta}=\hat{\theta}_{\mathscr{F}}\cup \hat{\theta}_{\mathscr{C}}
\end{equation}

Introducing the parameter sets $\theta_{\mathscr{F}}$ and $\theta_{\mathscr{F}}$ in relations (\ref{e:eq28}) we have:

\begin{equation}
\label{e:eq30}
\tilde{y} = \mathscr{C}(\theta_{\mathscr{C}}, v), \>\>v=\mathscr{F}(\theta_{\mathscr{F}}, x)
\end{equation}

Let $\mu$ denote the number of parameters in $\theta$ and $\mu_{\mathscr{C}}$ the number of parameters in $\theta_{\mathscr{C}}$. Without loss of generality we consider that an order exists for the parameters of the set $\theta$ such that the first $\mu_{\mathscr{C}}$ parameters are the parameters of set $\theta_{\mathscr{C}}$, i.e.,
$\theta^{\top}=(\theta_{\mathscr{C}}^{\top}\>\>\theta_{\mathscr{F}}^{\top})$. Given a training point
$z_i =(x_i, y_i), i \in[1, n]$, we denote by $\zeta_i = (v_i, y_i)$ a corresponding training point consisting of input feature vector $v_i = \mathscr{F}(\hat{\theta}_{\mathscr{F}}, x_i)$ and label $y_i$. We similarly define a training data set $\mathscr{Z}$ associated with the classifier $\mathscr{C}()$ as the set:

\begin{equation}
\label{e:eq31}
\mathscr{Z}=\{\zeta_i = (\mathscr{F}(\hat{\theta}_{\mathscr{F}}, x_i), y_i): (x_i, y_i) \in \mathcal{Z}, i \in [1,n]\}
\end{equation}

where set $\mathcal{Z}$ is the set used for training the complete system i.e., the featurizer-classifier pair. Set $\mathscr{Z}$ contains feature vector-label pairs, where feature vectors are the outputs the featurizer. Furthermore, the processing performed by the featurizer uses the optimal parameter value set  $\hat{\theta}_{\mathscr{F}}$. From the definition of set $\mathscr{Z}$ it follows that:

\begin{equation}
\label{e:eq32}
\mathcal{L}((x_i, y_i), \hat{\theta}) = \mathcal{L}((\mathscr{F}(x_i),y_i), \hat{\theta}_{\mathscr{C}}) \> \Leftrightarrow\> \mathcal{L}(z_i, \hat{\theta})=\mathcal{L}(\zeta_i, \hat{\theta}_{\mathscr{C}})
\end{equation}

Moreover, the Hessian matrices
$H_{\hat{\theta}} = \frac {1}{n} \sum_{i = 1}^{n} {\nabla_{\theta}^2 \> \mathcal{L}(z_i, \hat{\theta})}$ and
$H_{\hat{\theta}_{\mathscr{C}}} = \frac {1}{n} \sum_{i = 1}^{n} {\nabla_{\theta_{\mathscr{C}}}^2 \> \mathcal{L}(\zeta_i, \hat{\theta}_{\mathscr{C}})}$ satisfy:

\begin{equation}
\label{e:eq33}
H_{\hat{\theta}}= \begin{pmatrix} \>\>\> H_{\hat{\theta}_{\mathscr{C}}} & A \>\>\vspace{4pt}\\ A^\top & B\end{pmatrix}
\end{equation}

for some matrices $A \in M_{\mu_{\mathscr{C}}\times (\mu - \mu_{\mathscr{C}})}(\mathbb{R})$
and $B \in M_{ (\mu - \mu_{\mathscr{C}}) \times (\mu - \mu_{\mathscr{C}})}(\mathbb{R})$.
The following discussion \vspace{3pt}\\ focuses on Hessian matrices $H_{\hat{\theta}}$ for which their associated $H_{\hat{\theta}_{\mathscr{C}}}$ is positive definite.

\begin{defn}
\label{d:defn2}
(Classification Influence Function $I_{up}^{(\mathscr{C})}(\hat{\theta}_{\mathscr{C}}, \zeta_p)$).
Let $\mathcal{Z}$ be a training data set and $z_p = (x_p, y_p) \in \mathcal{Z}$ a training point. Let also $\hat{\theta} = \hat{\theta}_{\mathscr{F}} \cup  \hat{\theta}_{\mathscr{C}}$ be the parameter values of a machine learning system that minimize the empirical risk of relation (\ref{e:eq2}). The machine learning system consists of a featurizer and a classifier that satisfy relation (\ref{e:eq30}). Finally, let $\mathscr{Z}$ be the training data set defined by relation (\ref{e:eq31}), $\zeta_p$ be the point
$(\mathscr{F}(x_p), y_p)$, and $H_{\hat{\theta}_{\mathscr{C}}}$ the Hessian matrix related to $H_{\hat{\theta}}$ according to the right hand side of (\ref{e:eq33}). The classification influence function
$I_{up}^{(\mathscr{C})}(\hat{\theta}_{\mathscr{C}}, \zeta_p)$ is defined as:

\begin{equation}
\label{e:eq34}
I_{up}^{(\mathscr{C})}(\hat{\theta}_{\mathscr{C}}, \zeta_p) = -H_{\hat{\theta}_{\mathscr{C}}}^{-1}\cdot \nabla_{\theta_{\mathscr{C}}} \mathcal{L}(\zeta_p, \hat{\theta}_{\mathscr{C}})
\end{equation}
\end{defn}

\begin{defn}
\label{d:defn3}
(Classification Influence Function $I_{loss}^{(\mathscr{C})}(\hat{\theta}_{\mathscr{C}}, \zeta_t, \zeta_p)$).
Let $\mathcal{Z}$ be a training data set and $z_p = (x_p, y_p) \in \mathcal{Z}$ a training point. Let also $\hat{\theta} = \hat{\theta}_{\mathscr{F}} \cup  \hat{\theta}_{\mathscr{C}}$,
$\mathscr{Z}$, $\zeta_p$, and $H_{\hat{\theta}_{\mathscr{C}}}$ be as in Definition \ref{d:defn2}.  Finally, let $\zeta_t=(\mathscr{F}(x_t),y_t)$ be a test point associated with the feature vector-label pair $z_t = (x_t, y_t)$.  The classification influence function
$I_{loss}^{(\mathscr{C})}(\hat{\theta}_{\mathscr{C}}, \zeta_t, \zeta_p)$ is defined as:

\begin{equation}
\label{e:eq35}
I_{loss}^{(\mathscr{C})}(\hat{\theta}_{\mathscr{C}}, \zeta_t, \zeta_p) = - \nabla_{\theta_{\mathscr{C}}} \mathcal{L}(\zeta_t, \hat{\theta_{\mathscr{C}}})^\top \cdot H_{\hat{\theta}_{\mathscr{C}}}^{-1} \cdot \nabla_{\theta_{\mathscr{C}}} \mathcal{L}(\zeta_p, \hat{\theta}_{\mathscr{C}})
\end{equation}
\end{defn}

The classification influence functions $I_{up}^{(\mathscr{C})}(\hat{\theta}_{\mathscr{C}}, \zeta_p)$ and $I_{loss}^{(\mathscr{C})}(\hat{\theta}_{\mathscr{C}}, \zeta_t, \zeta_p)$ are not measures on the change in model parameters or loss experienced at $\zeta_t$ when the training point $\zeta_p$ is absent. This is because the classifier is not trained on the data set $\mathscr{Z}$. Instead, the classifier is trained as part of the complete featurizer-classifier pair using set $\mathcal{Z}$.

The purpose of introducing functions $I_{up}^{(\mathscr{C})}(\hat{\theta}_{\mathscr{C}}, \zeta_p)$ and  $I_{loss}^{(\mathscr{C})}(\hat{\theta}_{\mathscr{C}}, \zeta_t, \zeta_p)$ is to have lightweight approximations for the influence functions $I_{up}(\hat{\theta}, z_p)$ and $I_{loss}(\hat{\theta}, z_t, z_p)$, where such approximations may be much easier to compute. Indeed, if we assume that the number of parameters in the featurizer-classifier pair $\mu$ is in the order of hundreds of millions  and the number of parameters in the classifier $\mu_{\mathscr{C}}$ is in the order of thousands, then computing $I_{up}^{(\mathscr{C})}(\hat{\theta}_{\mathscr{C}}, \zeta_p)$ and  $I_{loss}^{(\mathscr{C})}(\hat{\theta}_{\mathscr{C}}, \zeta_t, \zeta_p)$  efficiently may be possible, whereas computing  $I_{up}(\hat{\theta}, z_p)$ and $I_{loss}(\hat{\theta}, z_t, z_p)$ may be infeasible due to the fact that the Hessian $H_{\hat{\theta}}$ and its inverse consist of quadrillions of elements.

In what follows we investigate under which assumptions functions $I_{up}^{(\mathscr{C})}(\hat{\theta}_{\mathscr{C}}, \zeta_p)$ and  $I_{loss}^{(\mathscr{C})}(\hat{\theta}_{\mathscr{C}}, \zeta_t, \zeta_p)$ are good approximations of $I_{up}(\hat{\theta}, z_p)$ and $I_{loss}(\hat{\theta}, z_t, z_p)$.

\begin{thm}
\label{t:thm1}
(On the Accuracy of Classification Influence Functions). Let $Q =\begin{pmatrix} Q_{\mathscr{C}} & \alpha\\ \beta & \gamma\end{pmatrix}$
be an orthonormal eigenbasis for the Hessian matrix $H_{\hat{\theta}}$, as well as its inverse  $H_{\hat{\theta}}^{-1}$, where
$Q_{\mathscr{C}} \in M_{\mu_{\mathscr{C}} \times \mu_{\mathscr{C}}}(\mathbb{R})$. Let $\{\lambda_1,\ldots,\lambda_{\mu_{\mathscr{C}}},\lambda_{\mu_{\mathscr{C}}+1},\ldots,\lambda_{\mu}\}$ be the set of eigenvalues of $H_{\hat{\theta}}$ corresponding to the eigenbasis $Q$. Finally, let $\Lambda_{\mathscr{C}}$ and $\Lambda_{\mathscr{F}}$ be the  diagonal matrices containing the eigenvalue sets $\{\lambda_1,\ldots,\lambda_{\mu_{\mathscr{C}}}\}$ and $\{ \lambda_{\mu_{\mathscr{C}}+1},\ldots,\lambda_{\mu}\}$
respectively, i.e.,

\begin{equation}
\label{e:eq36}
\Lambda_{\mathscr{C}}= \begin{pmatrix} \lambda_1 & \cdots & 0 \\ \vdots& \ddots &  \vdots\\ 0& \cdots& \lambda_{\mu_{\mathscr{C}}}\end{pmatrix},
\hspace{6pt}
\Lambda_{\mathscr{F}}= \begin{pmatrix} \lambda_{\mu_{\mathscr{C}}+1} & \cdots & 0 \\ \vdots& \ddots &  \vdots\\ 0& \cdots& \lambda_{\mu}\end{pmatrix}
\end{equation}

\noindent
such that:

\begin{equation}
\label{e:eq37}
\begin{array}{c}
H_{\hat{\theta}}= \begin{pmatrix} Q_{\mathscr{C}}^\top & \beta^\top \vspace{3pt}\\ \alpha^\top & \gamma^\top\end{pmatrix}  \begin{pmatrix} \Lambda_{\mathscr{C}} &  0 \vspace{3pt}\\ 0 & \Lambda_{\mathscr{F}}  \end{pmatrix} \begin{pmatrix} Q_{\mathscr{C}} & \alpha \vspace{3pt}\\ \beta & \gamma \end{pmatrix},
\hspace{6pt}
H_{\hat{\theta}}^{-1} = \begin{pmatrix} Q_{\mathscr{C}}^\top & \beta^\top \vspace{3pt}\\ \alpha^\top & \gamma^\top\end{pmatrix}  \begin{pmatrix} \Lambda_{\mathscr{C}}^{-1} &  0 \vspace{3pt}\\ 0 & \Lambda_{\mathscr{F}}^{-1}  \end{pmatrix} \begin{pmatrix} Q_{\mathscr{C}} & \alpha \vspace{3pt}\\ \beta & \gamma \end{pmatrix}
\vspace{4pt}\\
\end{array}
\end{equation}

\noindent
Then, if Corollaries \ref{c:cor1} and \ref{c:cor2} hold, the ratio $\displaystyle\frac{I_{loss}^{(\mathscr{C})}(\hat{\theta}_{\mathscr{C}}, \zeta_t, \zeta_p)}{I_{loss}(\hat{\theta}, z_t, z_p)}$ is given by:

\begin{equation}
\label{e:eq38}
\displaystyle\frac{I_{loss}^{(\mathscr{C})}(\hat{\theta}_{\mathscr{C}}, \zeta_t, \zeta_p)}{I_{loss}(\hat{\theta, z_t, z_p})} = \displaystyle\frac{g_t^\top (Q_{\mathscr{C}}^\top \Lambda_{\mathscr{C}}Q_{\mathscr{C}}\>+ \mathcal{T}_1)^{-1} g_p}{g_t^\top Q_{\mathscr{C}}^\top \Lambda^{-1}_{\mathscr{C}}Q_{\mathscr{C}}g_p \> + \> \mathcal{T}_2}
\end{equation}

\noindent
where:

\begin{equation}
\label{e:eq39}
g_t = \nabla_{\theta_{\mathscr{C}}}\mathcal{L}(\zeta_t, \hat{\theta}_{\mathscr{C}})
\end{equation}

\begin{equation}
\label{e:eq40}
g_p = \nabla_{\theta_{\mathscr{C}}}\mathcal{L}(\zeta_p, \hat{\theta}_{\mathscr{C}})
\end{equation}

\begin{equation}
\label{e:eq41}
\mathcal{T}_1 = \beta^\top \Lambda_{\mathscr{F}}\beta
\end{equation}

\begin{equation}
\label{e:eq42}
\mathcal{T}_2= g_t^\top \mathsf{F}_1 g_p+ g_t^\top \mathsf{F}_2f_p + f_t^\top \mathsf{F}_2^\top g_p + f_t^\top \mathsf{F}_3 f_p
\end{equation}

\begin{equation}
\label{e:eq43}
f_t = \nabla_{\theta_{\mathscr{F}}}\mathcal{L}(z_t, \hat{\theta})
\end{equation}

\begin{equation}
\label{e:eq44}
f_p = \nabla_{\theta_{\mathscr{F}}}\mathcal{L}(z_p, \hat{\theta})
\end{equation}

\begin{equation}
\label{e:eq45}
\mathsf{F}_1 =  \beta^\top \Lambda_{\mathscr{F}}^{-1} \beta
\end{equation}

\begin{equation}
\label{e:eq46}
\mathsf{F}_2 = Q_{\mathscr{C}}^\top \Lambda_{\mathscr{C}}^{-1} \alpha + \beta^\top \Lambda_{\mathscr{F}}^{-1} \gamma
\end{equation}

\begin{equation}
\label{e:eq47}
\mathsf{F}_3 = \alpha^\top \Lambda_{\mathscr{C}}^{-1} \alpha + \gamma^\top \Lambda_{\mathscr{F}}^{-1} \gamma
\end{equation}
\end{thm}

\begin{proof}
From the definition of vectors $\nabla_{\theta}\mathcal{L}(z_t, \hat{\theta})$,
$\nabla_{\theta}\mathcal{L}(z_p, \hat{\theta})$, $g_t$, $g_p$, $f_t$, and $f_p$, and from relation (\ref{e:eq32}) it follows that:

\begin{equation}
\label{e:eq48}
\nabla_{\theta}\mathcal{L}(z_t, \hat{\theta})^\top = (g_t^\top \>\> f_t^\top),
\hspace{6pt}
\nabla_{\theta}\mathcal{L}(z_p, \hat{\theta})^\top = (g_p^\top \>\> f_p^\top)
\end{equation}

\noindent
Next, from relation (\ref{e:eq37}), the inverse Hessian $H_{\hat{\theta}}^{-1}$ can be expressed as:

\begin{equation}
\label{e:eq49}
H_{\hat{\theta}}^{-1}=\begin{pmatrix} Q_{\mathscr{C}}^\top \Lambda_{\mathscr{C}}^{-1} Q_{\mathscr{C}} + \mathsf{F}_1 & \mathsf{F}_2 \> \\ \mathsf{F}_2^\top & \mathsf{F}_3 \> \end{pmatrix}
\end{equation}

\noindent
Also, from relations (\ref{e:eq33}) and (\ref{e:eq37}), it follows that:

\begin{equation}
\label{e:eq50}
H_{\hat{\theta}_{\mathscr{C}}}^{-1} = (Q_{\mathscr{C}}^\top \Lambda_{\mathscr{C}}Q_{\mathscr{C}}\>+ \mathcal{T}_1)^{-1}
\end{equation}

Theorem \ref{t:thm1} follows directly from combining relations (\ref{e:eq26}) and (\ref{e:eq35}) with relations
(\ref{e:eq39}), (\ref{e:eq40}), (\ref{e:eq42}), (\ref{e:eq48}), (\ref{e:eq49}) and (\ref{e:eq50}).

\end{proof}

From Theorem \ref{t:thm1} follows that, if the eigenvalues contained in the matrix $\Lambda_{\mathscr{C}}^{-1}$, i.e., $\{\frac{1}{\lambda_1},\ldots,\frac{1}{\lambda_{\mu_{\mathscr{C}}}}\}$, are the most dominant ones and if the elements of submatrix
$Q_{\mathscr{C}}$ have the highest values, then $\alpha, \beta, \gamma$ and $\Lambda_{\mathscr{F}}^{-1}$ are negligible for some specific set of bounds and norms. In this case, and for a subset of bounds and norms, the additive terms $\mathcal{T}_1$ and $\mathcal{T}_2$ can be negligible too.
This means that the ratio $\displaystyle\frac{I_{loss}^{(\mathscr{C})}(\hat{\theta}_{\mathscr{C}}, \zeta_t, \zeta_p)}{I_{loss}(\hat{\theta, z_t, z_p})}$ is equal to one plus some additive error, which depends on how small $\mathcal{T}_1$ and $\mathcal{T}_2$ are. Furthermore, the elements of $Q_{\mathscr{C}}$ are the eigenvector coordinates associated with the classifier parameters. Hence, under certain assumptions about the choice of classifier parameters, the classification  influence function $I_{loss}^{(\mathscr{C})}(\hat{\theta}_{\mathscr{C}}, \zeta_t, \zeta_p)$ is a good approximation of the influence function $I_{loss}(\hat{\theta}, z_t, z_p)$.

\section{Relative influence}
Even with the approximation of relation (\ref{e:eq35}), computing the Hessian $H_{\hat{\theta}_{\mathscr{C}}}$ may be expensive as the training data sets  $\mathcal{Z}$ or $\mathscr{Z}$ may be of very high cardinality. For example, $\mathcal{Z}$ or $\mathscr{Z}$ may contain billions of training points. In this section we introduce the notion of relative influence functions, which are influence functions computed on Hessian matrices coming from small subsets of the training data sets. We demonstrate that, under certain assumptions about the choice of points, the replacement of the Hessian with one coming from a smaller set of points does not impact the information conveyed by the computed influence values.

\vspace{4pt}
\begin{defn}
\label{d:defn4}
(Hessian Matrix Relative to a Subset of a Training  Data Set). Let $\mathcal{Z} = \{(x_i, y_i), i \in [1,n]\}$ be a training data set, $R(\theta)= \displaystyle \frac {1}{n} \displaystyle \sum_{i = 1}^{n} {\mathcal{L}(z_i, \theta)}$ an empirical risk associated with loss function $\mathcal{L}()$, $\hat{\theta}$ set of model parameter values that minimize this empirical risk, and  $H_{\hat{\theta}} = \displaystyle \frac {1}{n} \displaystyle \sum_{i = 1}^{n} {\nabla_{\theta}^2 \> \mathcal{L}(z_i, \hat{\theta})}$ a Hessian matrix associated with the data set, loss function and optimal parameter values. Let also $\mathcal{Z}^{(m)}=\{(x_{i_j}, y_{i_j}) \in \mathcal{Z}, j \in [1,m], \> i_j \in [1, n] \>\> \forall \> j \in [1, m]\}$ be a subset of $\mathcal{Z}$ of cardinality $m \leq n$.  The Hessian matrix relative to the subset $\mathcal{Z}^{(m)}$ of $\mathcal{Z}$ is defined as the matrix:

\begin{equation}
\label{e:eq51}
H_{\hat{\theta}}^{(m)} = \displaystyle \frac {1}{m} \displaystyle \sum_{j= 1}^{m} {\nabla_{\theta}^2 \> \mathcal{L}((x_{i_j}, y_{i_j}),\hat{\theta})}
\end{equation}
\end{defn}

\noindent
The relative Hessian matrix $H_{\hat{\theta}}^{(m)}$ is computed only on the training points of subset $\mathcal{Z}^{(m)}$ and not on the entire data set $\mathcal{Z}$. In the analysis that follows we assume that every relative Hessian matrix considered is positive definite.

\vspace{4pt}
\begin{defn}
\label{d:defn5} (Relative Influence Functions). Let $\mathcal{Z}$ and $\mathcal{Z}^{(m)}$ be a training data  set and a subset of it, as in Definition \ref{d:defn4}. Let also  $H_{\hat{\theta}}$ and $H_{\hat{\theta}}^{(m)}$ be the Hessian and relative Hessian matrices associated with sets $\mathcal{Z}$ and $\mathcal{Z}^{(m)}$, respectively. A relative influence function $I_{up}^{(m)}(\hat{\theta}, z_p)$ associated with the matrix $H_{\hat{\theta}}^{(m)}$ is defined as:

\begin{equation}
\label{e:eq52}
I_{up}^{(m)}(\hat{\theta}, z_p)=-{H_{\hat{\theta}}^{(m)}}^{-1}\cdot \nabla_{\theta}\mathcal{L}(z_p, \hat{\theta})
\end{equation}

where $z_p = (x_p, y_p)$ is a training point in $\mathcal{Z}$. Similarly, a relative influence function $I_{loss}^{(m)}(\hat{\theta}, z_t, z_p)$ associated with $H_{\hat{\theta}}^{(m)}$ and test point $z_t = (x_t, y_t)$ is  defined as:

\begin{equation}
\label{e:eq53}
I_{loss}^{(m)}(\hat{\theta}, z_t, z_p)=- \nabla_{\theta}\mathcal{L}(z_t, \hat{\theta})^\top \cdot {H_{\hat{\theta}}^{(m)}}^{-1}\cdot \nabla_{\theta}\mathcal{L}(z_p, \hat{\theta})
\end{equation}
\end{defn}

Of particular interest is the problem of determining the relationship between the relative influence function $I_{loss}^{(m)}(\hat{\theta}, z_t, z_p)$ and function $I_{loss}(\hat{\theta}, z_t, z_p)$, when the test point $z_t$ is fixed, $z_p$ can be any training point in $\mathcal{Z}$, and Corollaries \ref{c:cor1} and \ref{c:cor2} hold. Specifically, we consider that $I_{loss}(\hat{\theta}, z_t, z_p)$ is given by expression (\ref{e:eq26}), and that any negligible additive terms in $I_{loss}(\hat{\theta}, z_t, z_p)$, which are omitted in (\ref{e:eq26}), do not impact the correctness of the relations that follow.

\vspace{4pt}
\begin{thm}
\label{t:thm2}
(On Loss Estimate Preserving (LEP) Relative Influence Functions).
Let $\mathcal{Z}$  and $\mathcal{Z}^{(m)}$ be training data sets as in Definition \ref{d:defn5}.  Let $H_{\hat{\theta}}$ and $H_{\hat{\theta}}^{(m)}$ be Hessian and relative  Hessian matrices associated with $\mathcal{Z}$  and $\mathcal{Z}^{(m)}$. Let $I_{loss}(\hat{\theta}, z_t, z_p)$ and $I_{loss}^{(m)}(\hat{\theta}, z_t, z_p)$ be influence  and relative influence functions given by expressions (\ref{e:eq26}) and (\ref{e:eq53}), and $z_t$ a fixed test point, the loss gradient of which is not orthogonal to the loss gradient of any of the training points in $\mathcal{Z}$. Finally, we consider the set $\mathcal{R}$ defined as:

\begin{equation}
\label{e:eq54a}
\mathcal{R} = \{r_p=\displaystyle\frac{I_{loss}(\hat{\theta}, z_t, z_p)}{I_{loss}^{(m)}(\hat{\theta}, z_t, z_p)}, z_p \in \mathcal{Z}\}
\end{equation}

If there exists a value $\delta \in \mathbb{R}$, $ \delta \neq 0$, for which the following two conditions are true:

\begin{itemize}
    \item for every $r \in \mathcal{R}$, the matrix $\mathcal{N}(r, \delta, H_{\hat{\theta}}, H_{\hat{\theta}}^{(m)}, z_t) = \delta\cdot {H_{\hat{\theta}}^{(m)}}G(\delta, z_t)^{-1}\cdot(H_{\hat{\theta}}^{-1}-rH_{\hat{\theta}}^{(m)^{-1}})$ is negligible when added to $r \cdot I$, and when both $\mathcal{N}(r, \delta, H_{\hat{\theta}}, H_{\hat{\theta}}^{(m)}, z_t)$ and $r\cdot I$ are multiplied with $G(\delta, z_t)$ from the left and $\nabla_{\theta}\mathcal{L}(z_{p}, \hat{\theta})$ from the right
    \item for every training point $z_p$ and corresponding ratio $ r_p=\displaystyle\frac{I_{loss}(\hat{\theta}, z_t, z_p)}{I_{loss}^{(m)}(\hat{\theta}, z_t, z_p)}$, every scalar product $\mathsf{p}_{\delta,p}$ computed between $\delta$ and any of the elements of the gradient vector $\nabla_{\theta}\mathcal{L}(z_{p,} \hat{\theta})$ is negligible when  added to scalar  $\mathsf{p}_{t, p} = \nabla_{\theta}\mathcal{L}(z_t, \hat{\theta})^\top \cdot \nabla_{\theta}\mathcal{L}(z_{p}, \hat{\theta})$, and when the entities $\mathsf{p}_{\delta, p}, \mathsf{p}_{t, p}$ are each multiplied with  one of $r_p$  or an element of matrix $G(\delta, z_t) \cdot H_{\hat{\theta}}^{(m)} \cdot H_{\hat{\theta}}^{-1} \cdot G(\delta, z_t)^{-1}$
\end{itemize}

then:

\begin{equation}
\label{e:eq54}
\displaystyle \frac{I_{loss}^{(m)}(\hat{\theta}, z_t, z_p)}{I_{loss}(\hat{\theta}, z_t, z_p)} = \rho^{-1}, \> \forall z_p \in \mathcal{Z}
\end{equation}

for some $\rho \in \mathbb{R}$, where equality (\ref{e:eq54}) omits negligible additive terms, which may depend on the training point $z_p$. The term ``negligible'' is defined in Definition \ref{d:defn1}, $I$ is the $\mu \times \mu$ identity matrix, $\mu$ is the number of parameters in $\hat{\theta}$, and:

\begin{equation}
\label{e:eq55}
G(\delta, z_t)=\begin{pmatrix} \nabla_{\theta}\mathcal{L}(z_t, \hat{\theta})^\top  \\ \vdots \\ \nabla_{\theta}\mathcal{L}(z_t, \hat{\theta})^\top  \end{pmatrix} + \delta \cdot I
\end{equation}

\end{thm}

Since $\delta \neq 0$, matrix $G(\delta, z_t)$ is invertible. Theorem \ref{t:thm2} suggests that for specific $\mathcal{Z}$, $H_{\hat{\theta}}$, and test point $z_t$, there exists a class of subsets $\mathcal{Z}^{(m)}$ and associated relative Hessian matrices $H_{\hat{\theta}}^{(m)}$, which impact the loss change computations on every training point of $\mathcal{Z}$ in a uniform manner, if these computations are performed using the relative influence function $I_{loss}^{(m)}()$ and not $I_{loss}()$. Furthermore, the search for a specific subset $\mathcal{Z}^{(m)}$ and matrix $H_{\hat{\theta}}^{(m)}$ of this class could be realized with reasonable complexity. We discuss this issue more below. We refer to the relative influence functions of this class as ``Loss Estimate Preserving'' (LEP) relative influence functions.

\begin{proof}
We assume that the hypothesis holds. We will show that relation (\ref{e:eq54}) is true. Let $z_p$ be a specific training point for which:

\begin{equation}
\label{e:eq56}
\displaystyle \frac{I_{loss}^{(m)}(\hat{\theta}, z_t, z_{p})}{I_{loss}(\hat{\theta}, z_t, z_{p})} = r^{-1}
\end{equation}

for some specific ratio $r = r(z_p) \in \mathcal{R}$ that depends on $z_p$ and fixed test point $z_t$ from the hypothesis. We also consider the matrix
$G^{*}(z_t) = G(\delta, z_t) - \delta \cdot I$ and $g$ the gradient vector $\nabla_{\theta}\mathcal{L}(z_{p}, \hat{\theta})$. Relation (\ref{e:eq56}) can be written as:

\begin{equation}
\label{e:eq57}
\nabla_{\theta}\mathcal{L}(z_t, \hat{\theta})  H_{\hat{\theta}}^{-1}  g = r  \nabla_{\theta}\mathcal{L}(z_t, \hat{\theta})  H_{\hat{\theta}}^{(m)^{-1}}  g \Rightarrow G^{*}(z_t) H_{\hat{\theta}}^{-1}  g = r G^{*}(z_t)  H_{\hat{\theta}}^{(m)^{-1}}  g
\end{equation}

Adding the term $t_1 = \delta H_{\hat{\theta}}^{-1}  g + r  \delta H_{\hat{\theta}}^{(m)^{-1}}g$
to both sides of (\ref{e:eq57}) we obtain:

\begin{equation}
\label{e:eq58}
G(\delta, z_t) H_{\hat{\theta}}^{-1}  g = r  G(\delta, z_t) H_{\hat{\theta}}^{(m)^{-1}}  g + \delta (H_{\hat{\theta}}^{-1} g - r  H_{\hat{\theta}}^{(m)^{-1}} g)
\end{equation}

Next we multiply both sides of (\ref{e:eq58}) with the factor
$f = G(\delta, z_t) \cdot H_{\hat{\theta}}^{(m)}\cdot G(\delta, z_t)^{-1}$. This results in:

\begin{equation}
\label{e:eq59}
(G(\delta, z_t)  H_{\hat{\theta}}^{(m)} H_{\hat{\theta}}^{-1}G(\delta, z_t)^{-1}) \cdot G(\delta, z_t)g = G( \delta, z_t)\cdot (r I + \mathcal{N}(r, \delta, H_{\hat{\theta}}, H_{\hat{\theta}}^{(m)}, z_t)) \cdot g
\end{equation}

According to the first condition of the hypothesis, the term $t_2 = \mathcal{N}(r, \delta, H_{\hat{\theta}}, H_{\hat{\theta}}^{(m)}, z_t)$  is negligible when added to $r \cdot I$. We notice that both terms are multiplied from the left and from the right with the factors mentioned in the first condition of the hypothesis. We also follow the convention of omitting negligible terms in equality relations. Hence:

\begin{equation}
\label{e:eq60}
(G(\delta, z_t)  H_{\hat{\theta}}^{(m)} H_{\hat{\theta}}^{-1}G(\delta, z_t)^{-1}) \cdot G(\delta, z_t)g = r \cdot G( \delta, z_t) g
\end{equation}

We proceed with the proof denoting the elements of the $\mu \times \mu$ matrix $G(\delta, z_t) \cdot H_{\hat{\theta}}^{(m)}H_{\hat{\theta}}^{-1}\cdot G(\delta, z_t)^{-1}$ by $\{ w_{ij}, i \in [i, \mu], j \in [1, \mu]\}$, the dot product $\nabla_{\theta}\mathcal{L}(z_t, \hat{\theta})^\top \cdot g$ by $\mathsf{p}$ and the elements of vector $g$ by $\{y_i, i \in [1, \mu]\}$. Equality (\ref{e:eq60}) is thus written as:

\begin{equation}
\label{e:eq61}
\begin{pmatrix}  w_{11}& \cdots &  w_{1\mu}\\ \vdots& \ddots & \vdots \\ w_{\mu 1} & \cdots & w_{\mu \mu} \end{pmatrix} \begin{pmatrix}  \mathsf{p} + \delta y_1 \\  \vdots \\ \mathsf{p} + \delta y_{\mu} \end{pmatrix} =r \cdot  \begin{pmatrix}  \mathsf{p} + \delta y_1 \\  \vdots \\ \mathsf{p} + \delta y_{\mu} \end{pmatrix}
\end{equation}

According to the second condition of the hypothesis the terms $\delta y_i, i \in [1, \mu]$ are negligible when added to $\mathsf{p}$. We notice that terms $\delta y_i, i \in [1, \mu]$ and $\mathsf{p}$ are multiplied with the factors mentioned in the hypothesis. Hence:

\begin{equation}
\label{e:eq62}
\begin{pmatrix}  w_{11}& \cdots &  w_{1\mu}\\ \vdots& \ddots & \vdots \\ w_{\mu 1} & \cdots & w_{\mu \mu} \end{pmatrix} \begin{pmatrix}  \mathsf{p}  \\ \vdots \\ \mathsf{p} \end{pmatrix} =r \cdot  \begin{pmatrix}  \mathsf{p} \\  \vdots \\ \mathsf{p} \end{pmatrix}
\end{equation}

Equating the corresponding elements of the vectors of the left and right hand sides of (\ref{e:eq62}), we obtain:

\begin{equation}
\label{e:eq63}
(w_{11} + \ldots+w_{1 \mu}) \cdot\mathsf{p} = r \cdot \mathsf{p}, \ldots, (w_{\mu 1} + \ldots+w_{\mu \mu}) \cdot\mathsf{p} = r \cdot \mathsf{p}
\end{equation}

Scalar $\mathsf{p}$ is nonzero, according to the assumption that the loss gradient in $z_t$ is not orthogonal to the loss gradient in any training point in $\mathcal{Z}$. Hence, scalar $\mathsf{p}$ can be eliminated from the equations of (\ref{e:eq63}):

\begin{equation}
\label{e:eq64}
r = w_{11} + \ldots+w_{1 \mu}, \ldots, r = w_{\mu 1} + \ldots+w_{\mu \mu}
\end{equation}

The equations of (\ref{e:eq64}) indicate that the sum of the elements of the rows of matrix $G(\delta, z_t) \cdot H_{\hat{\theta}}^{(m)}H_{\hat{\theta}}^{-1}\cdot G(\delta, z_t)^{-1}$ is the same for all rows, with the exception of negligible additive terms, and that the value or $r$ is independent of the training point $z_p$. Denoting the sum of the elements of a row of $G(\delta, z_t) \cdot H_{\hat{\theta}}^{(m)}H_{\hat{\theta}}^{-1}\cdot G(\delta, z_t)^{-1}$ by $\rho$, it must hold that $r = \rho$ and Theorem \ref{t:thm2} is proven.

\end{proof}

Since  the conditions of the hypothesis concern only gradient vectors from the finite set $\mathcal{Z}$ and not vectors in any infinite subset of $\mathbb{R}^{\mu}$, we conjecture that the set  of matrices $H_{\hat{\theta}}^{(m)}$ which satisfy these conditions has non-negligible cardinality and members of the set can be identified by executing sequences of feasible computation steps. These computations are outlined below. These differ from identifying trivial sets that satisfy Theorem \ref{t:thm2}, like for instance when matrix $H_{\hat{\theta}}^{(m)}$ has one eigenvector in common with $H_{\hat{\theta}}$. This matrix satisfies Theorem \ref{t:thm2} when all gradient vectors coming from the training points of $\mathcal{Z}$ are parallel to this common eigenvector.

\vspace{4pt}
In what follows we informally discuss a sequence of steps that can be used for identifying LEP relative influence functions and their associated relative Hessian matrices. These steps do not require necessarily computing the exact Hessian $H_{\hat{\theta}}$. First, the analytical expressions for the loss function $\mathcal{L}()$, as well as domain knowledge about the coordinates of the points in $\mathcal{Z}$, is taken into account in order to obtain properties for the matrices $\nabla_{\theta}^2\mathcal{L}(z_p, \hat{\theta}), \> z_p \in \mathcal{Z}$.  Properties may include  upper and lower bounds, or distribution parameters for the elements of matrices $\nabla_{\theta}^2\mathcal{L}(z_p, \hat{\theta})$. From these properties, coarser approximations for the elements of the \vspace{-1pt} Hessian $H_{\hat{\theta}}$ and its inverse are obtained, as well as for the values in the set $\mathcal{R}$. Next, the \vspace{-2pt}space of alternative pairs $(\delta, H_{\hat{\theta}}^{(m)})$ is searched, for example by means of randomly sampling. If a value $\delta$
\vspace{-2pt}and a matrix $H_{\hat{\theta}}^{(m)}$ are found such that the matrix $t_3 = \delta\cdot {H_{\hat{\theta}}^{(m)}}G(\delta, z_t)^{-1}$,
when multiplied with an \vspace{-2pt} approximation of $t_4 = r^{-1} (H_{\hat{\theta}}^{-1}-rH_{\hat{\theta}}^{(m)^{-1}})$, produces a negligible additive matrix, then $H_{\hat{\theta}}^{(m)}$ must be an LEP relative Hessian \vspace{-2pt}matrix, with very high probability, and its associated $I_{loss}^{(m)}()$ an LEP relative influence function.

\vspace{4pt}
The procedure could be performed once to determine the
range of values for the cardinality of the set
$\mathcal{Z}^{(m)}$ resulting in LEP relative Hessian matrices. Once the expected cardinality of the set $\mathcal{Z}^{(m)}$ is determined, we could compute relative matrices $H_{\hat{\theta}}^{(m)}$ using lighter methods such as by directly sampling as many points as the expected cardinality for $\mathcal{Z}^{(m)}$ twice, computing a pair of relative Hessian
matrices, and observing uniform impact on the influence function  between the two matrices.

\begin{figure}
  \centering
  \includegraphics[scale=0.37]{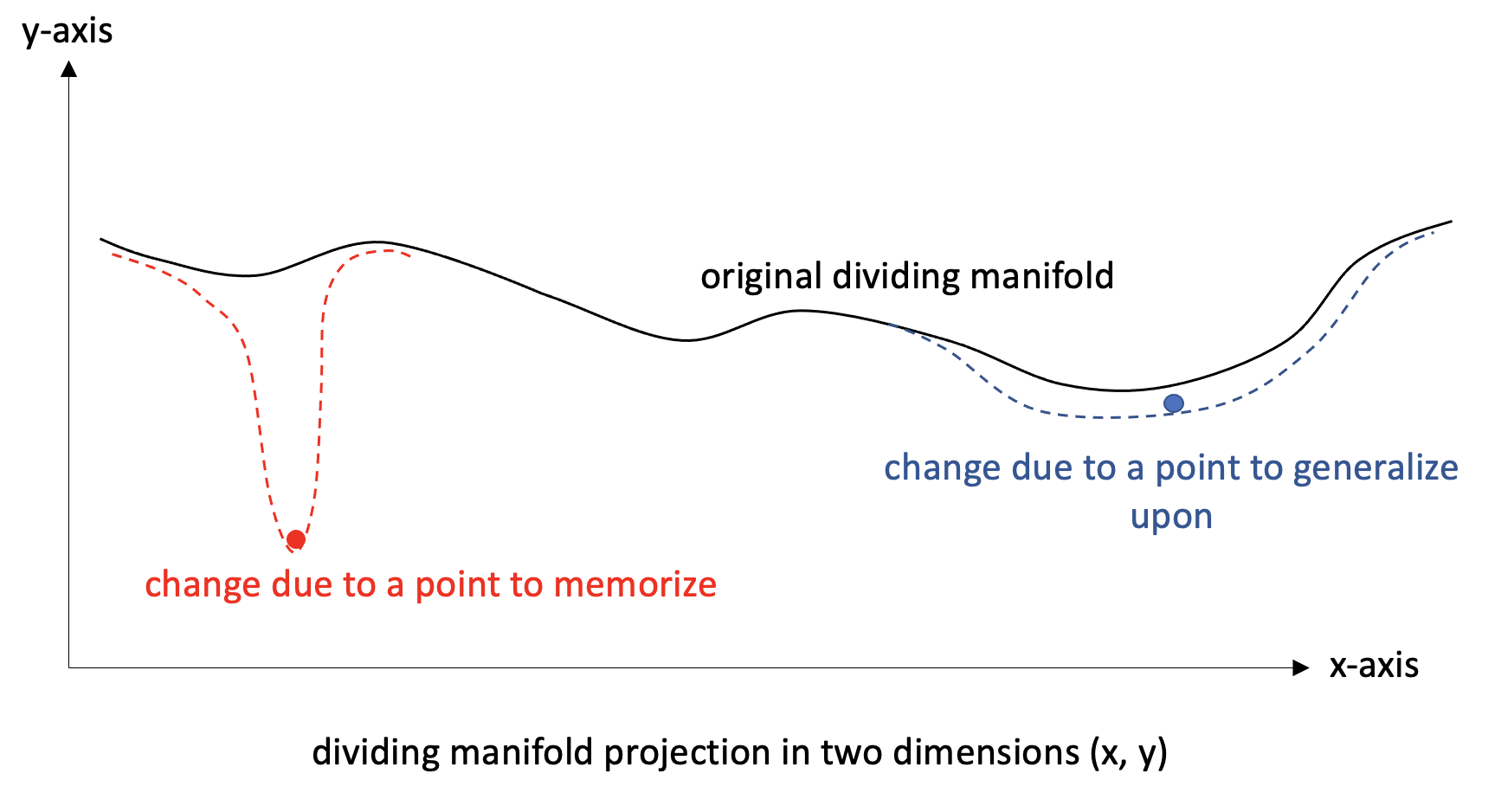}
  \caption{Changes to the class dividing manifold caused by a training point to memorize, and a training point to generalize upon}
\label{p:pic2}
\end{figure}

\section{Memorization and generalization}

One interesting distinction concerning training points is \vspace{2pt} the distinction between points to memorize and points to generalize upon. Memorizable \vspace{2pt}  training points are typically those which are not part of any dense spatial cluster. Memorizable points are isolated \vspace{2pt}points, located at sparse areas of the multi-dimensional space in which training points are \vspace{2pt}represented. Memorizing such points typically results in significant changes in the neural network model parameters \vspace{2pt}and class dividing manifold. One the other hand, points to generalize upon are found in dense spatial \vspace{2pt}clusters. Furthermore, the addition of a generalizable point in the training data \vspace{2pt}set results in modest changes in the model parameters and dividing manifold. In what follows we formally \vspace{2pt}define what a memorizable and a generalizable point is.

\vspace{4pt}
\begin{defn}
\label{d:defn6}
($D$-memorizable training point). Let $\mathcal{Z}$ be a training data set as in relation (\ref{e:eq1}), $z_p \in \mathcal{Z}$ a training point in this data set, $\hat{\theta}$ a set of optimal model parameters given by relation (\ref{e:eq3}), and $H_{\hat{\theta}}$ a positive definite Hessian matrix associated with $\hat{\theta}$ and loss function $\mathcal{L}()$ as in relation (\ref{e:eq4}). Let $D \in \mathbb{R}$  be a threshold value. Training point $z_p$ is called $D$-memorizable if and only if:

\begin{equation}
\label{e:eq65}
\| H_{\hat{\theta}}^{-1} \nabla_{\theta}\mathcal{L} (z_p, \hat{\theta})\|_1 > D
\end{equation}

where $\|\cdot\|_1$ is the $\mathbb{L}_1$ norm and vectors $\nabla_{\theta}\mathcal{L} (z_p, \hat{\theta})$, $H_{\hat{\theta}}^{-1} \nabla_{\theta}\mathcal{L} (z_p, \hat{\theta})$ are represented in the space defined by the orthonormal eigenbasis of $H_{\hat{\theta}}$. This is also the orthonormal eigenbasis of $H_{\hat{\theta}}^{-1}$.
\end{defn}

\vspace{4pt}
\begin{defn}
\label{d:defn7}
($D$-generalizable training point).
Let $\mathcal{Z}$ be a training data set, $z_p \in \mathcal{Z}$ a training point in this data set, $\hat{\theta}$, $H_{\hat{\theta}}$ and  $D \in \mathbb{R}$ as in Definition \ref{d:defn6}. Training point $z_p$ is called $D$-generalizable if and only if:

\begin{equation}
\label{e:eq66}
\| H_{\hat{\theta}}^{-1} \nabla_{\theta}\mathcal{L} (z_p, \hat{\theta})\|_1 \leq D
\end{equation}

where again $\|\cdot\|_1$ is the $\mathbb{L}_1$ norm and the vectors $\nabla_{\theta}\mathcal{L} (z_p, \hat{\theta})$, $H_{\hat{\theta}}^{-1} \nabla_{\theta}\mathcal{L} (z_p, \hat{\theta})$ are represented in the space defined by the orthonormal eigenbasis of $H_{\hat{\theta}}$.
\end{defn}

In the next Theorem, we show that the sign of the influence function $I_{loss}(\hat{\theta}, z_t, z_p)$ given by (\ref{e:eq26}) indicates whether a training point is memorizable, for some test point $z_t$.

\vspace{4pt}
\begin{thm}
\label{t:thm3}
(On a training  point being memorizable).
Let $\mathcal{Z}$ be a training data set, $z_p \in \mathcal{Z}$ a training point in this data set, $\hat{\theta}$ and $z_p \in \mathcal{Z}$ as in Definition \ref{d:defn6}, $z_t$ some test point with non-zero gradient $\nabla_{\theta}\mathcal{L} (z_t, \hat{\theta})$, and $I_{loss}(\hat{\theta}, z_t, z_p)$ an influence function computed on $H_{\hat{\theta}}, z_t, z_p$, and given by relation (\ref{e:eq26}). Let $\{\chi_1,\ldots,\chi_{\mu}\}$ and $\{\psi_1,\ldots,\psi_{\mu}\}$ be the representations of the vectors $\nabla_{\theta}\mathcal{L} (z_p, \hat{\theta})$ and $\nabla_{\theta}\mathcal{L} (z_t, \hat{\theta})$ in the space defined by the orthonormal eigenbasis of $H_{\hat{\theta}}$, respectively, and $\{\lambda_1,\ldots,\lambda_{\mu}\}$ the corresponding eigenvalues of $H_{\hat{\theta}}^{-1}$, which are positive values. Let also $\lambda_{min}$ be the minimum of $\lambda_1,\ldots,\lambda_{\mu}$, and  $T^{+}(z_p, z_t)$ and $T^{-}(z_p, z_t)$ threshold functions defined as:

\begin{equation}
\label{e:eq67}
\begin{array} {c}
T^{+}(z_p, z_t) = \displaystyle\sum_{i = 1}^{\mu} \chi_i \psi_i 1(\chi_i\psi_i \geq0)
\vspace{4pt}\\
T^{-}(z_p, z_t)=-\displaystyle\sum_{i=1}^{\mu}\chi_i\psi_i1(\chi_i\psi_i < 0)
\end{array}
\end{equation}

where $1(\textsf{expression})=1$ if $\textsf{expression} = \textsf{true}$  or $0$ otherwise. Finally let $\mathcal{T}^{+}$ and $\mathcal{T}^{-}$ be defined as:

\begin{equation}
\label{e:eq68}
\begin{array} {c}
\mathcal{T}^{+} = \inf_{t \in \{T^{+}(z_p, z_t), z_p \in \mathcal{Z}\}}\{t\}
\vspace{4pt}\\
\mathcal{T}^{-} = \inf_{t \in \{T^{-}(z_p, z_t), z_p \in \mathcal{Z}\}}\{t\}
\end{array}
\end{equation}

We also assume that Corollaries \ref{c:cor1} and \ref{c:cor2} hold.  Then the following statements are true:

\begin{itemize}
    \item If $I_{loss}(\hat{\theta}, z_t, z_p) > 0$ then training point $z_p$ is
    $D^{+}$-memorizable, where $D^{+} = \displaystyle\frac{\mathcal{T}^{+}\lambda_{min}}{\max_{\psi\in\{|\psi_1|,\ldots,|\psi_{\mu}|\}}\psi}$
    \item If $I_{loss}(\hat{\theta}, z_t, z_p) < 0$ then training point $z_p$ is
    $D^{-}$-memorizable, where $D^{-} = \displaystyle\frac{\mathcal{T}^{-}\lambda_{min}}{\max_{\psi\in\{|\psi_1|,\ldots,|\psi_{\mu}|\}}\psi}$
\end{itemize}
\end{thm}

\begin{proof}
We first show that the first of the two statements is true. Let $Q = (q_1  \ldots q_{\mu})$ be the orthonormal eigenbasis of  $H_{\hat{\theta}}$. Then it holds that:

\begin{equation}
\label{e:eq69}
\begin{array} {c}
I_{loss}(\hat{\theta}, z_t, z_p) > 0 \Rightarrow (\psi_1q_1^\top+\ldots+ \psi_{\mu}q_{\mu}^\top)\cdot(\lambda_1\chi_1q_1+\ldots+\lambda_{\mu}\chi_{\mu}q_{\mu})<0 \Rightarrow
\vspace{4pt}\\
\displaystyle\sum_{i=1}^{\mu}\lambda_i \chi_i\psi_i <0 \Rightarrow \displaystyle\sum_{i=1}^{\mu}\lambda_i|\chi_i\psi_i|1(\chi_i\psi_i < 0) > \displaystyle\sum_{i=1}^{\mu}\lambda_i\chi_i\psi_i1(\chi_i\psi_i \geq0)
\end{array}
\end{equation}

Relation (\ref{e:eq69}) can be further written as:

\begin{equation}
\label{e:eq70}
\begin{array} {c}
\displaystyle\sum_{i=1}^{\mu}\lambda_i|\chi_i|
 \cdot {\max_{\psi\in\{|\psi_1|,\ldots,|\psi_{\mu}|\}}\psi} \geq \displaystyle\sum_{i=1}^{\mu} \displaystyle\lambda_i|\chi_i\psi_i|1(\chi_i\psi_i < 0) > \displaystyle\sum_{i=1}^{\mu}\lambda_i\chi_i\psi_i1(\chi_i\psi_i \geq0) \Rightarrow
\vspace{4pt}\\
\displaystyle\sum_{i=1}^{\mu}\lambda_i|\chi_i| > \displaystyle\frac{ \displaystyle\sum_{i=1}^{\mu}\lambda_i\chi_i\psi_i1(\chi_i\psi_i \geq0)}{\max_{\psi\in\{|\psi_1|,\ldots,|\psi_{\mu}|\}}\psi} \Rightarrow \displaystyle\sum_{i=1}^{\mu}\lambda_i|\chi_i| > \displaystyle\frac{ \mathcal{T}^{+}\lambda_{min}}{\max_{\psi\in\{|\psi_1|,\ldots,|\psi_{\mu}|\}}\psi}
\end{array}
\end{equation}

The first statement of Theorem \ref{t:thm3} follows directly from relation (\ref{e:eq70}) and from the observation that:

\begin{equation}
\label{e:eq71}
\| H_{\hat{\theta}}^{-1} \nabla_{\theta}\mathcal{L} (z_p, \hat{\theta})\|_1 = \displaystyle\sum_{i=1}^{\mu}|\lambda_i||\chi_i| =  \displaystyle\sum_{i=1}^{\mu}\lambda_i|\chi_i|
\end{equation}

The second statement of Theorem \ref{t:thm3} is proven in a similar manner:

\begin{equation}
\label{e:eq71b}
\begin{array} {c}
I_{loss}(\hat{\theta, z_t, z_p}) < 0 \Rightarrow \displaystyle\sum_{i=1}^{\mu}\lambda_i \chi_i\psi_i1(\chi_i\psi_i \geq 0) > \displaystyle\sum_{i=1}^{\mu}\lambda_i|\chi_i\psi_i|1(\chi_i\psi_i < 0) \Rightarrow
\vspace{4pt}\\
\displaystyle\sum_{i=1}^{\mu}\lambda_i|\chi_i| > \displaystyle\frac{ -\displaystyle\sum_{i=1}^{\mu}\lambda_i\chi_i\psi_i1(\chi_i\psi_i  < 0)}{\max_{\psi\in\{|\psi_1|,\ldots,|\psi_{\mu}|\}}\psi} \Rightarrow \displaystyle\sum_{i=1}^{\mu}\lambda_i|\chi_i| > \displaystyle\frac{ \mathcal{T}^{-}\lambda_{min}}{\max_{\psi\in\{|\psi_1|,\ldots,|\psi_{\mu}|\}}\psi}
\end{array}
\end{equation}

This, together with relation (\ref{e:eq71}) completes the proof of the second statement of Theorem \ref{t:thm3}.

\end{proof}

\section{On the correlation between the influence value sign  and the label}

The sign of the influence function not only indicates whether a point is generalizable of memorizable but also whether a training point is beneficial or harmful to a specific test point. One should expect that, given a test data set, some memorizable points may increase the loss experienced  at some test points. This means that memorizing certain training points may come at the cost of reducing the accuracy of predictions in some points of the test data set. In what follows we study this behavior in the context of system consisting of a featurizer and a classifier where the classifier consists of a single linear layer and the loss function is the binary cross entropy coupled with the sigmoid transformation.

\vspace{4pt}
\begin{defn}
\label{d:defn8}
(Single linear layer classifier).
Let $\mathscr{F}()$ and $\mathscr{C}()$ be a featurizer and a classifier as in Figure \ref{p:pic1}, where the response of the combined featurizer-classifier system is given by equation (\ref{e:eq28}) and the system parameters satisfy equations (\ref{e:eq29}) and (\ref{e:eq30}). Classifier is called single linear layer classifier if its response $\mathscr{C}(u)$ is given by:

\begin{equation}
\label{e:eq72}
\mathscr{C}(u)=\mathscr{C}((u_0 ...u_{\mathsf{D}-1}))=\mathsf{sigm}\big(\displaystyle\sum_{i=0}^{\mathsf{D}-1}w_i u_i + b\big)
\end{equation}

where $\mathsf{sigm}(x)\leftarrow\displaystyle\frac{1}{1+ e^{-x}}$ is the sigmoid function.
\end{defn}

\vspace{4pt}
Relation (\ref{e:eq72}) can also be written as:

\begin{equation}
\label{e:eq73}
\mathscr{C}(u)= \displaystyle\frac{1}{1 + e^{-f(u)}}, \> f(u)=\displaystyle\sum_{i=0}^{\mathsf{D}-1} w_i u_i + b
\end{equation}

The set of parameters $\theta_{\mathscr{C}}$ of this classifier consists of a weight vector $(w_0...w_{\mathsf{D}-1})$ of size $\mathsf{D}$ and a bias $b$:

\begin{equation}
\label{e:eq74}
\theta_{\mathscr{C}}=\{w_0, w_1,...,w_{\mathsf{D}-1}, b\}, \> \text{and} \>
\hat{\theta}_{\mathscr{C}}=\{\hat{w}_0, \hat{w}_1,...,\hat{w}_{\mathsf{D}-1}, \hat{b}\}
\end{equation}

The analysis that follows applies to the binary cross entropy loss function. Let\rq s consider a training or test point
$\zeta \leftarrow \{ (u_0...u_{\mathsf{D}-1}), y\}$. The binary cross entropy loss function $\mathcal{L}(\zeta, \theta_{\mathscr{C}})$, associated with single linear layer parameters $\theta_{\mathscr{C}}$, which accepts input $\zeta$ is defined as:

\begin{equation}
\label{e:eq75}
\mathcal{L}(\zeta, \theta_{\mathscr{C}})= -\big(\>y\cdot \mathsf{ln}(\mathscr{C}((u_0...u_{\mathsf{D}-1}))) \>+\> (1-y)\cdot\mathsf{ln}(1-\mathscr{C}((u_0...u_{\mathsf{D}-1}))) \>\big)
\end{equation}

where $\mathsf{ln}()$ is the natural logarithm. Before stating the main theorem that establishes the correlation between influence values and labels, we provide a number of useful lemmas:

\vspace{4pt}
\begin{lem}
\label{l:lem1}
(On the first and second partial derivatives of the binary \vspace{3pt} cross entropy loss function).
Let $\zeta \leftarrow \{u, y\}$ and $\mathcal{L}(\zeta, \theta_{\mathscr{C}})$ be an input point and loss function as in relation (\ref{e:eq75}). Let also the parameter set $\mathcal{L}(\zeta, \theta_{\mathscr{C}})$ be as in relation (\ref{e:eq74}) and $f(u) \leftarrow \sum_{i=0}^{\mathsf{D}-1} w_i u_i + b$. Then the following relations hold:

\begin{equation}
\label{e:eq76}
\displaystyle\frac{\partial \mathcal{L}(\zeta, \theta_{\mathscr{C}})}{\partial f(u)}=  -y + \displaystyle\frac{1}{1 +e^{-f(u)}}
\end{equation}

\begin{equation}
\label{e:eq77}
\displaystyle\frac{\partial^2 \mathcal{L}(\zeta, \theta_{\mathscr{C}})}{\partial f(u)^2}= \displaystyle\frac{e^{-f(u)}}{(1 + e^{-f(u)})^2}
\end{equation}

\begin{equation}
\label{e:eq78}
\displaystyle\frac{\partial \mathcal{L}(\zeta, \theta_{\mathscr{C}})}{\partial w_i}=  \big(-y + \displaystyle\frac{1}{1 +e^{-f(u)}}\big) \cdot u_i, \>\>
\displaystyle\frac{\partial \mathcal{L}(\zeta, \theta_{\mathscr{C}})}{\partial b}=  -y + \displaystyle\frac{1}{1 +e^{-f(u)}}
\end{equation}

and

\begin{equation}
\label{e:eq79}
\displaystyle\frac{\partial^2 \mathcal{L}(\zeta, \theta_{\mathscr{C}})}{\partial w_i \partial w_j}= \displaystyle\frac{e^{-f(u)}}{(1 + e^{-f(u)})^2} \cdot u_i u_j, \>\>
\displaystyle\frac{\partial^2 \mathcal{L}(\zeta, \theta_{\mathscr{C}})}{\partial w_i \partial b}= \displaystyle\frac{e^{-f(u)}}{(1 + e^{-f(u)})^2} \cdot u_i
\end{equation}

where $i, j \in[0, \mathsf{D}-1]$.
\end{lem}

\vspace{4pt}
\begin{proof}
It holds that:

\begin{equation}
\label{e:eq80}
\begin{array}{l}
\displaystyle\frac{\partial \mathcal{L}(\zeta, \theta_{\mathscr{C}})}{\partial f(u)}=  \Big(-y\cdot \mathsf{ln}\big(\displaystyle\frac{1}{1 + e^{-f(u)}}\big)\Big)'+\Big( (y-1)\cdot\mathsf{ln}\big( \displaystyle\frac{e^{-f(u)}}{1 + e^{-f(u)}}\big)\Big)'
\vspace{4pt}\\
\hspace{47pt}
= - y\cdot \displaystyle\frac{e^{-f(u)}}{1 + e^{-f(u)}}-(y-1) \cdot \displaystyle\frac{1}{1 + e^{-f(u)}}
\vspace{4pt}\\
\hspace{47pt}
= \displaystyle\frac{-ye^{-f(u)}-y+1}{1 + e^{-f(u)}} = -y + \displaystyle\frac{1}{1+e^{-f(u)}}
\end{array}
\end{equation}

which completes the proof of the first statement of Lemma \ref{l:lem1}. The second statement can be shown in a similar manner, using relation (\ref{e:eq80}):

\begin{equation}
\label{e:eq81}
\displaystyle\frac{\partial^2 \mathcal{L}(\zeta, \theta_{\mathscr{C}})}{\partial f(u)^2} =  \Big( -y + \displaystyle\frac{1}{1 + e^{-f(u)}}\Big)'= \displaystyle\frac{e^{-f(u)}}{(1 + e^{-f(u)})^2}
\end{equation}

Finally the last two statements of the lemma follow directly from relations (\ref{e:eq80}) and (\ref{e:eq81}) and the chain rule of differentiation.

\end{proof}

The correlation between influence and labels is established from Lemma \ref{l:lem1} and the Lemma and Definition that follow.

\vspace{4pt}
\begin{lem}
\label{l:lem2}
(On the sign of the gradient of the loss function).
Let $\zeta \leftarrow \{u, y\}$ be a point such that
$u_i \geq 0 \> \forall i \in [0, \mathsf{D}-1]$, and $\mathcal{L}(\zeta, \theta_{\mathscr{C}})$ a loss function as in relation (\ref{e:eq75}). Finally, let the parameter set
$\theta_{\mathscr{C}}$ be as in relation (\ref{e:eq74}) and $\mathsf{sign}()$ the sign function defined by $\mathsf{sign}(x) = 1(x \geq 0) - 1(x < 0)$. Then, for every element $g_i, i \in [0, \mathsf{D}] $ of the gradient vector $\nabla_{\theta_{\mathscr{C}}}\mathcal{L}(\zeta, \theta_{\mathscr{C}}) $ it holds that:

\begin{equation}
\label{e:eq82}
\mathsf{sign}(g_i)= 1(y=0) -1(y=1)
\end{equation}
\end{lem}

\vspace{4pt}
\begin{proof}
From the chain rule of differentiation, and for $i \in [0, \mathsf{D}-1]$, it holds that:

\begin{equation}
\label{e:eq83}
g_i = \displaystyle\frac{\partial \mathcal{L}(\zeta, \theta_{\mathscr{C}})}{\partial f(u)} \cdot\displaystyle\frac{\partial f(u)}{w_i}
\end{equation}

Furthermore, from the definition of the classifier as a single linear layer it holds that:

\begin{equation}
\label{e:eq84}
\displaystyle\frac{\partial f(u)}{w_i}=u_i
\end{equation}

Since it holds that $u_i \geq 0$ from the hypothesis, the sign of $g_i$ is given by:

\begin{equation}
\label{e:eq85}
\mathsf{sign}(g_i)=\mathsf{sign}\big(\displaystyle\frac{\partial \mathcal{L}(\zeta, \theta_{\mathscr{C}})}{\partial f(u)}\big)= \mathsf{sign}\big(-y + \displaystyle\frac{1}{1 +e^{-f(u)}} \big)
\end{equation}

When $y=1$, $\mathsf{sign}(g_i)=\mathsf{sign}(-\displaystyle\frac{e^{-f(u)}}{1 + e^{-f(u)}})=-1$. When  $y=0$, $\mathsf{sign}(g_i)=\mathsf{sign}(\displaystyle\frac{1}{1 + e^{-f(u)}})=1$.  Hence, the lemma is proven for  $i \in [0, \mathsf{D}-1]$. The proof is similar for the last element of $g$.

\end{proof}

\vspace{4pt}
\begin{defn}
\label{d:defn9}
(Inverse Hessian matrix with almost positive elements).
Let $\mathscr{F}()$, $\mathscr{C}()$ be a featurizer-classifier pair as in the Theorem \ref{t:thm1} with the classifier $\mathscr{C}()$ consisting of $\mathsf{D}+1$ model parameters, $\mathcal{Z}$ a training data set with
$\mathscr{Z}$ being its associated set of points passed as input to the classifier, $\mathscr{Z}_g$ a set of points referred to as the ``gradient set'', and $\mathcal{L}()$ a loss function. The inverse Hessian matrix $H_{\theta_{\mathscr{C}}}^{-1}$ defined by:

\begin{equation}
\label{e:eq86}
H_{\theta_{\mathscr{C}}}^{-1} = \begin{pmatrix} \displaystyle\frac{1}{|\mathscr{Z}|} \displaystyle\sum_{\zeta^{(p)} \leftarrow \{u^{(p)}, y^{(p)}\} \in \mathscr{Z}} \begin{pmatrix} \displaystyle\frac{\partial^2 \mathcal{L}(\zeta^{(p)}, \theta_{\mathscr{C}})}{\partial w_0^2}) & \cdots & \displaystyle\frac{\partial^2 \mathcal{L}(\zeta^{(p)}, \theta_{\mathscr{C}})}{\partial w_0 \partial b } \\ \vdots & \ddots & \vdots \\  \displaystyle\frac{\partial^2 \mathcal{L}(\zeta^{(p)}, \theta_{\mathscr{C}})}{\partial b \partial w_0 } & \cdots & \displaystyle\frac{\partial^2 \mathcal{L}(\zeta^{(p)}, \theta_{\mathscr{C}})}{\partial b^2}  \end{pmatrix} \end{pmatrix}^{-1}
\end{equation}

\vspace{4pt}
is called inverse Hessian with almost positive elements with respect  to gradient set $\mathscr{Z}_g$ and loss function $\mathcal{L}()$ if each of the elements of the matrix $h_{i,j}, i, j \in [0, \mathsf{D}]$ satisfies:

\begin{equation}
\label{e:eq87}
\mathsf{sign}(h_{i,j} + \delta_{i,j})=1
\end{equation}

where the sum of the additive terms $\delta_{i,j}$ is negligible when each of the additive terms $\delta_{i,j}$ is multiplied by any element  of gradient value $\nabla_{\theta_{\mathscr{C}}}\mathcal{L}(\zeta_g^{(0)}, \theta_{\mathscr{C}})$, and any element of gradient value $\nabla_{\theta_{\mathscr{C}}}\mathcal{L}(\zeta_g^{(1)}, \theta_{\mathscr{C}})$, for any $\zeta_g^{(0)}, \zeta_g^{(1)} \in \mathscr{Z}_g$.
\end{defn}

\vspace{4pt}
Having stated and proven Lemmas \ref{l:lem1} and \ref{l:lem2}, and having provided Definition \ref{d:defn9}, we proceed with stating the main theorem of this section, which establishes the correlation that exists between influence values and labels under certain assumptions.

\vspace{4pt}
\begin{thm}
\label{t:thm4}
(On the correlation that exists between influence functions and labels). Let
$\mathscr{F}()$, $\mathscr{C}()$ be a featurizer-classifier pair as in the Theorem \ref{t:thm1}, for which Corollaries \ref{c:cor1} and \ref{c:cor2} hold, the additive terms $\mathcal{T}_1$ and $\mathcal{T}_2$ of relation (\ref{e:eq38}) are negligible, the classifier $\mathscr{C}()$ is a single layer linear classifier as in Definition \ref{d:defn8}, and the assumptions of Lemma \ref{l:lem2} hold. Furthermore, let the loss function under consideration $\mathcal{L}()$ be the binary cross entropy loss function given by relation (\ref{e:eq75}). Finally, let $\mathcal{Z}$ be a training data set of cardinality $n$ and $\mathscr{Z}$ its associated set of points passed as input to the classifier, given by relation (\ref{e:eq31}). Then, for every training point $z_p \leftarrow \{x_p, y_p\}\in \mathcal{Z}$ and its associated $\zeta_p \leftarrow \{\mathscr{F}(x_p), y_p\}\in \mathscr{Z}$, and for every test point $z_t \leftarrow \{x_t, y_t\}$ and its associated $\zeta_t \leftarrow \{\mathscr{F}(x_t), y_t\}$, such that the matrix $H_{\theta_{\mathscr{C}}}^{-1}$ is an inverse Hessian matrix with almost positive elements with respect to gradient set $\mathscr{Z}_g \leftarrow \mathscr{Z} \> \cup \> \{\zeta_t\}$ and loss function $\mathcal{L}()$, the following hold:

\begin{equation}
\label{e:eq88}
\mathsf{sign}(I_{loss}(\hat{\theta}, z_t, z_p))=1(y_t \neq y_p)-1(y_t = y_p)
\end{equation}

\begin{equation}
\label{e:eq89}
\mathcal{L}(z_t, \hat{\theta}_{-\frac{1}{n}, z_p}) \geq \mathcal{L}(z_t, \hat{\theta})
\>\>\> \normalfont\text{if} \> y_t = 0 \>
\normalfont\text{and} \> y_p=0
\end{equation}

\begin{equation}
\label{e:eq90}
\mathcal{L}(z_t, \hat{\theta}_{-\frac{1}{n}, z_p}) < \mathcal{L}(z_t, \hat{\theta})
\>\>\> \normalfont\text{if} \> y_t = 0 \> \normalfont\text{and} \> y_p=1
\end{equation}

\begin{equation}
\label{e:eq91}
\mathcal{L}(z_t, \hat{\theta}_{-\frac{1}{n}, z_p}) < \mathcal{L}(z_t, \hat{\theta})
\>\>\> \normalfont\text{if} \> y_t = 1 \>
\normalfont\text{and} \> y_p=0
\end{equation}

\begin{equation}
\label{e:eq92}
\mathcal{L}(z_t, \hat{\theta}_{-\frac{1}{n}, z_p}) \geq \mathcal{L}(z_t, \hat{\theta})
\>\>\> \normalfont\text{if} \> y_t = 1 \>
\normalfont\text{and} \> y_p=1
\end{equation}

where in the above relations, negligible additive terms have been omitted.
\end{thm}

\vspace{4pt}
\begin{proof}
We observe that statements (\ref{e:eq89})-(\ref{e:eq92}) follow directly from \vspace{2pt}statement (\ref{e:eq88}), the assumption that Corollaries \ref{c:cor1} and \ref{c:cor2} hold and relation (\ref{e:eq25}). Therefore, to prove \vspace{2pt}Theorem \ref{t:thm4}, we need to show that statement (\ref{e:eq88}) holds.
Since the additive terms $\mathcal{T}_1$ and $\mathcal{T}_2$ of relation (\ref{e:eq38}) are negligible, the influence function $I_{loss}^{(\mathscr{C})}(\hat{\theta}_{\mathscr{C}}, \zeta_t, \zeta_p)$ is a good
\vspace{2pt}approximation of the influence function $I_{loss}(\hat{\theta}, z_t, z_p)$. Omitting negligible additive terms, we write that:

\begin{equation}
\label{e:eq93}
\begin{array} {l}
\mathsf{sign}(I_{loss}(\hat{\theta}, z_t, z_p))= \mathsf{sign}(I_{loss}^{(\mathscr{C})}(\hat{\theta}_{\mathscr{C}}, \zeta_t, \zeta_p))
\vspace{4pt}\\
\hspace{85pt}
=-\mathsf{sign}(\nabla_{\theta_{\mathscr{C}}}\mathcal{L}(\zeta_t, \theta_{\mathscr{C}}) \cdot H_{\theta_{\mathscr{C}}}^{-1}  \cdot \nabla_{\theta_{\mathscr{C}}}\mathcal{L}(\zeta_p, \theta_{\mathscr{C}}) )
\end{array}
\end{equation}

We continue with the proof denoting the elements of $\nabla_{\theta_{\mathscr{C}}}\mathcal{L}(\zeta_p, \theta_{\mathscr{C}})$ by $g_i^{(p)}$, the elements of $\nabla_{\theta_{\mathscr{C}}}\mathcal{L}(\zeta_t, \theta_{\mathscr{C}})$ by $g_i^{(t)}$, and the elements of the
\vspace{2pt}inverse Hessian $H_{\theta_{\mathscr{C}}}^{-1}$ by $h_{i,j}$, where $i, j \in[0, \mathsf{D}]$. Relation (\ref{e:eq93}) is written as:

\begin{equation}
\label{e:eq94}
\begin{array} {l}
\mathsf{sign}(I_{loss}(\hat{\theta}, z_t, z_p))=-\mathsf{sign}\Big(\displaystyle\sum_{i=0}^{\mathsf{D}} \> \displaystyle\sum_{j=0}^{\mathsf{D}} \> g^{(t)}_i h_{i,j} g_j^{(p)} \Big)
\vspace{4pt}\\
\hspace{85pt}
= -\mathsf{sign}\Big(\displaystyle\sum_{i=0}^{\mathsf{D}} \> \displaystyle\sum_{j=0}^{\mathsf{D}} \> g^{(t)}_i \cdot (h_{i,j} + \delta_{i,j})\cdot g_j^{(p)} \Big)
\end{array}
\end{equation}

The terms $\delta_{i,j}$ are introduced in relation (\ref{e:eq93}) because, as stated, their sum is negligible when each term is multiplied with gradient elements $g_i^{(t)}$,  $g_j^{(p)}$. Next we examine the sign of each of the terms $ g^{(t)}_i \cdot (h_{i,j} + \delta_{i,j})\cdot g_j^{(p)} $ that appear in relation (\ref{e:eq94}). Factor $(h_{i,j} + \delta_{i,j})$ is \vspace{3pt}always non-negative from the assumption about the inverse Hessian matrix. On the other hand, from Lemma \ref{l:lem2} it holds that, if labels $y_t, y_p$ are equal then the  \vspace{2pt}product $g_i^{(t)}g_j^{(p)}$ has sign equal to
$1$ , and if they are different has sign equal to $-1$. Hence:

\begin{equation}
\label{e:eq95}
 \mathsf{sign}(g^{(t)}_i \cdot (h_{i,j} + \delta_{i,j})\cdot g_j^{(p)})=1(y_t = y_p)-1(y_t \neq y_p) \> \forall i, j \in [0, \mathsf{D}]
\end{equation}

The statement (\ref{e:eq88}) follows directly by combining relations (\ref{e:eq94}) and (\ref{e:eq95}). This completes the proof of Theorem \ref{t:thm4}.

\end{proof}

\section{Discussion}
We have shown that influence functions computed on the parameters of a classifier can be good approximations of the standard theory influence functions. This is true when the eigenvalues of the inverse Hessian, associated with the classifier parameters, are the most dominant ones from among all inverse Hessian eigenvalues,  and the same applies to the corresponding eigenvectors. Moreover, the sign of the influence value indicates whether a training point is to memorize, as opposed to generalize upon. Thus, keeping the training points to memorize and sampling from the training points to generalize may be a good sampling methodology resulting in accuracy improvement.

If for a machine learning system the assumptions of Theorem \ref{t:thm4} hold, then removing negative samples improves the accuracy of the predictions on the positives but degrades the accuracy of the predictions on the negatives. The opposite is also true. This means that given a system accuracy metric such as Area-Under-the-Curve (AUC) or average Binary Cross Entropy (BCE) loss, one can optimize for this metric by  adjusting the ratio of positives over negatives. Investigating the validity of the assumptions of Theorem \ref{t:thm4} experimentally is the subject of future work.

\section*{References}

\medskip

\small
[1] P. W. Koh and P. Liang, ``Understanding Black Box Predictions with Influence Functions'', arXiv:1703.04730

[2]  Christopher R. Palmer and Christos Faloutsos, ``Density biased sampling: An improved method for data mining and clustering'', SIGMOD Rec., 29(2):82–92, May 2000. ISSN 0163-5808. doi: 10.1145/335191.335384

[3] Daniel Ting and Eric Brochu, ``Optimal subsampling with influence functions'', In S. Bengio, H. Wallach, H. Larochelle, K. Grauman, N. Cesa-Bianchi, and R. Garnett, editors, Advances in Neural Information Processing Systems, volume 31. Curran Associates, Inc., 2018.

[4] Zifeng Wang, Hong Zhu, Zhenhua Dong, Xiuqiang He, and Shao-Lun Huang, ``Less is better: Unweighted data subsampling via influence functions'', In Proceedings of the AAAI Conference on Artificial Intelligence (AAAI), 2020

[5] Zifeng Wang, Rui Wen, Xi Chen, Shao-Lun Huang, Ningyu Zhang, Yefeng Zheng, ``Finding Influential Instances for Distantly Supervised Relation Extraction'', arxiv.org/abs/2009.09841

[6] F. R. Hampel, ``The influence curve and its role in robust
estimation'', Journal of the American Statistical Association, 69(346):383–393, 1974.

\end{document}